\documentclass[10pt]{article}

\usepackage{amsthm} 

\newtheorem{theorem}{Theorem}

\usepackage{graphicx}
\usepackage{amsmath,amsfonts,amssymb,bm}
\usepackage{algorithm, algorithmic}
\usepackage[usenames,dvipsnames]{xcolor}
\usepackage{xspace}
\usepackage{varwidth}
 \usepackage{multirow}
 \usepackage{hyperref}
 
\usepackage[
 citestyle=authoryear,
 bibstyle=authoryear,
 isbn=false,
 natbib=true,
 backend=bibtex,
 firstinits=true,
 sorting=nyt,
 dashed=false,
 maxbibnames=99]{biblatex}

\DeclareFieldFormat
  [article,inbook,incollection,inproceedings,patent,thesis,unpublished]
  {title}{#1}
\DeclareFieldFormat[article]{number}{(#1)}
\DeclareNameAlias{sortname}{last-first}

\renewbibmacro{in:}{%
  \ifentrytype{article}{%
  }{%
    \printtext{\bibstring{in}\intitlepunct}%
  }%
}

\renewbibmacro*{volume+number+eid}{%
  \printfield{volume}%
  \printfield{number}%
  \setunit{\addcomma\space}%
  \printfield{eid}}

\title{Learning to Predict Independent of Span}

\allowdisplaybreaks

\def\beq{\begin{equation}}
\def\eeq{\end{equation}}

\def\a{\alpha}
\def\b{\beta}
\def\d{\delta}
\def\l{\lambda}
\def\tr{^\top}
\def\g{\gamma}
\def\th{\bm\theta}
\def\onth{\bm{\tilde\th}}
\def\tth{\bm{\theta}}

\def\p{\bm\phi}

\def\E#1{{\mathbb E}\!\left[#1\right]}

\def\P#1{{\mathbb P}\!\left(#1\right)}

\def\aa{\bm a}

\def\ee{\bm e}
\def\el{\ee}
\def\egl{\ee}

\def\t#1#2{\th_{#1}^{#2}}
\def\ont#1#2{\bm{\tilde{\theta}}_{#1}^{#2}}
\def\tt#1#2{\bm{\theta}_{#1}^{#2}}

\def\ZZ#1#2{\Z_{#1}^{#2}}
\def\Zb#1#2{\Z_{#1}^{#2}}

\def\Zl#1#2{\Z_{#1}^{#2}}
\def\Zlg#1#2{\Z_{#1}^{#2}}
\def\Zblg#1#2{\bar{\Z}_{#1}^{#2}}

\def\tau{T}
\def\P{P} 
\def\F{{\bf F}} 
\def\I{{\bf I}}
\def\R{X}
\def\Z{Z}

\def\der_emph#1{#1}

\author{
  Hado van Hasselt \thanks{Google DeepMind} \and Richard S. Sutton \thanks{Reinforcement Learning and Artificial Intelligence Laboratory\newline \-\hspace{.5cm} Department of Computing Science, University of Alberta\newline \-\hspace{.5cm} Edmonton, Alberta, Canada T6G 2E8}}

\begin{document}
\maketitle

\begin{abstract}%
We consider how to learn multi-step predictions efficiently.
Conventional algorithms wait until observing actual outcomes before performing the computations to update their predictions.
If predictions are made at a high rate or span over a large amount of time, substantial computation can be required to store all relevant observations and to update all predictions when the outcome is finally observed.  We show that the exact same predictions can be learned in a much more computationally congenial way, with uniform per-step computation that does not depend on the span of the predictions.  We apply this idea to various settings of increasing generality, repeatedly adding desired properties and each time deriving an equivalent span-independent algorithm for the conventional algorithm that satisfies these desiderata.
Interestingly, along the way several known algorithmic constructs emerge spontaneously from our derivations, including dutch eligibility traces, temporal difference errors, and averaging.  This allows us to link these constructs one-to-one to the corresponding desiderata, unambiguously connecting the `how' to the `why'.  Each step, we make sure that the derived algorithm subsumes the previous algorithms, thereby retaining their properties.  Ultimately we arrive at a single general temporal-difference algorithm that is applicable to the full setting of reinforcement learning.
\end{abstract}

\section{Learning long-term predictions}


The \emph{span} of a multi-step prediction is the number of steps elapsing between when the prediction is made and when its target or ideal value is known. We consider the case in which predictions are made repeatedly, at each of a sequence of discrete time steps. For example, if on each day we predict what a stock market index will be in 30 days, then the span is 30, whereas if we predict at each hour what the stock market index will be in 30 days, then the span is $30 \times 24 = 720$.

The span may vary for individual predictions in a sequence. For example, if we predict on each day what the stock-market index will be at the end of the year, then the span will be much longer for predictions made in January than it is for predictions made in December. If the span may vary in this way, then we consider the span of the prediction sequence to be the maximum possible span of any individual prediction in the sequence. For example, the span of a daily end-of-year stock-index prediction is 365. Often the span is infinite. For example, in reinforcement learning we often learn value functions that are predictions of the discounted sum of all future rewards in the potentially infinite future  \parencite{SuttonBarto:1998}.

In this paper we consider computational and algorithmic issues in efficiently learning long-term predictions, defined as predictions of large integer span. Predictions could be long term in this sense either because a great deal of clock time passes, as in predicting something at the end of the year, or because predictions are made very often, with a short time between steps (e.g., as in high-frequency financial trading).
The per-step computational complexity of some algorithms for learning accurate predictions depends on the span of the predictions, and this can become a significant concern if the span is large. Therefore, we focus on the construction of learning algorithms whose computational complexity per time step (in both time and memory) is constant (does not scale with time) and independent of span.

This paper features two recurring themes, the first of which is the repeated spontaneous emergence of, often well-known,  algorithmic constructs, directly from our derivations.  We start each derivation by formalizing a desired property and constructing an algorithm that fulfills it, without considering computationally efficiency.  Then, we derive a span-independent algorithm that results on each step in exactly the same predictions.  Interestingly, each time a specific algorithmic construct emerges, demonstrating a clear connection between the desideratum (the `why') and the algorithmic construct (the `how'). For instance, the desire to be independent of span leads to a dutch eligibility trace, which was previously derived only in the more specific context of online temporal difference (TD) learning \citep{vanSeijen:2014}.

The second theme is that we unify the algorithms at each step.  Each time, we make sure to obtain an algorithm that is strictly more general than the previous ones, so that in the end we obtain one single algorithm that can fulfill all the desiderata while remaining computationally congenial.

\section{Outline of the paper}
In this section, we briefly describe the high-level narrative of the paper, without going into technical detail.  In each of the Sections 3 to 8, we describe and formalize one or more desirable properties for our algorithms and then derive a computationally congenial algorithm that achieves this exactly.  We build up to the final, most general, algorithm that is ultimately derived in Section 8 to highlight the connections between desired properties and algorithmic constructs. Making these connections clear is one of the main goals of this paper.

Specifically, in Section \ref{sec:span} we derive a span-independent algorithm to update the predictions for a single final outcome. The algorithm is \emph{offline} in the sense that does not change its predictions before observing the outcome.  The dutch trace emerges spontaneously, which shows that this trace is closely tied to the requirement of span-independent computation.  This emergence is surprising and intriguing because it shows that these traces are not specific to online TD learning, for which they were first proposed \parencite{vanSeijen:2014}.

In Section \ref{sec:online} we derive span-independent updates that update the predictions \emph{online}, towards interim targets that temporarily stand in for the final outcome.  We show that the desire to be online results in the spontaneous emergence of TD errors \parencite{Sutton:1984,Sutton:1988}.  
In this paper we are mostly agnostic to the origin of the interim targets.  These may for instance be given by external experts or by own online predictions, as in standard TD learning \parencite[e.g., see][]{SuttonBarto:1998}.

It can be beneficial to be able to switch smoothly between online and offline updates, on a step-by-step basis, for instance when we do not full trust some of the interim targets that we would use for our online updates.  This allows us to have the best of both worlds: the online predictions stay trustworthy even if some interim targets are wrong, and we are still able to use any useful information immediately when it is observed.  In Section \ref{sec:trust} we consider how to do this efficiently and from our derivation an update emerges that averages the online weights in a separate trusted weight vector.  This is interesting because such averaging is known to improve the convergence rates of online learning algorithms \parencite{Polyak:1992,Bach:2013}, but seems to only rarely be used in reinforcement learning \parencite[as noted, e.g., by][]{Szepesvari:2010}.

Some interim targets may be so informative that we want their effect to persist in the predictions even after observing the final outcome.  For instance, if the final outcome is stochastic and the interim targets are drawn independently from the same distribution it makes sense to average these instead of committing fully only to the final outcome.  In the extreme, we might see an interim target that we trust so much that we do not even care about the actual outcome anymore, for instance because the interim target already takes into account all possible outcomes from that moment rather than only the specific one that will happen to materialize this time, resulting in a more accurate prediction on average than a single final outcome.  In Section \ref{sec:persistence}, we formalize these ideas and show they lead naturally to a form of TD($\l$) \parencite{Sutton:1988,SuttonBarto:1998}.

The $\l$ parameter that governs the amount of persistency of the interim targets can be interpreted as representing a degree of trust: if we trust an interim target fully ($\l = 0$) we do not need to consider later observations, while if we distrust it fully ($\l=1$) it will be replaced by later targets and leave no trace in the final predictions.  This is a different notion of trust than the one considered for the smooth switching between online and offline updates, where the trust was relative to the actual final outcome rather than the expected outcome.  These two forms of trust are compatible and complimentary, and in Section \ref{sec:persistent_trust} we show how to combine them into a single algorithm.

Up to Section \ref{sec:persistent_trust}, we have only considered predicting a single final outcome in an episodic setting. In Section \ref{sec:generalizations} we consider how to deal with two important generalizations of the problem setting: cumulative returns, and soft terminations.  Cumulative returns allow us to see part of the return on each step, and allow us to start learning from these partial returns immediately in the online setting.  Soft terminations allow us to learn about predictions that may conditionally terminate even if the actual process continues, and they allow for non-episodic predictions that may terminal softly on each step rather than completely at a single point in time.  This leads to a single final algorithm that subsumes all previous algorithms as special cases.  The algorithm is similar to the conventional TD($\l$) algorithm but with important differences that ensure that it is exactly equivalent to the desired, but inefficient, algorithm and therefore inherits all its desirable properties.

Because our final algorithm is novel, it is appropriate to analyze it.  In Section \ref{sec:analysis} we prove that the algorithm is convergent under typical mild conditions, and that it converges to the same solution as similar previous algorithms, including TD($\l$).

We conclude the paper with a short discussion in Section \ref{sec:discussion}.

\section{Independence of span and the emergence of traces}\label{sec:span}

We start with a supervised learning setting---predicting the final numeric outcome of an episodic process. An episode of the process starts at time $t=0$ and moves stochastically from state to state generating feature vectors $\p_t$ until termination with a final numeric outcome $\Z$ at final time $\tau$. For example, $\Z$ could be the price of a particular stock that we want to predict, and each episode may be a year, such that time $\tau$ corresponds to the end of the year.

We consider the general case of multi-step predictions ($\tau > 1$), where a prediction is made on each step. The standard supervised learning setting is a special case where in each episode we only make one prediction (such that, without loss of generality, we can take $\tau = 1$). 

Our predictions are linear, that is, the prediction at time $t$ is the inner product of $\p_t$ and a learned weight vector $\th$, denoted $\p_t\tr\th$. The algorithms are indifferent to the origins of the features, which may be handcrafted or learned.\footnote{This includes, for instance, the case where $\p_t$ is the last hidden layer of a neural network.}  The weights have an initial value $\th_0$ that is presumably due to previous episodes. We analyze how the weights change in a single episode (and thus we do not include the episode number in our notation).

At the final time, when $Z$ is observed, we can update all the predictions towards the target as in the classical least mean squares (LMS) algorithm defined by the updates:
\begin{equation}\label{FV_tau}
\th_{t+1} \doteq \th_t + \a_t \p_t \left( \Z - \p_t\tr\th_t \right) ,  ~~~~~t=0,\ldots,\tau-1, 
\end{equation}
where $\a_t>0$ is a step-size parameter that may vary from time step to time step (e.g., as a function of the state at that time).
We call this a \emph{forward view}, because to update the prediction at time $t$ we need to look forward in time to the outcome $\Z$ which is observed at the later time $\tau$.

To perform the updates \eqref{FV_tau} we have to wait until $\Z$ is known and then do the update for all previous time steps $t$.  This requires storing and then computing updates for all the preceding feature vectors. The required computational resources scale with the span of the prediction (the maximum length of an episode), which is what we wish to avoid. We seek incremental computations whose per-time-step complexity is $O(n)$, where $n$ is the number of parameters, and that result in the same weights as \eqref{FV_tau} at the end of the episode. That is, the incremental updates should compute the same $\th_{\tau}$ as \eqref{FV_tau} if they are given the same input (the same $\th_0$, the same sequence $\{\p_t\}_{t=0}^{\tau-1}$, and the same $\Z$).

It may seem that the best we can hope for is to approximate the result computed by the LMS algorithm, because of the strict computational restriction.  Such a trade off between computation and accuracy is not uncommon.  We will however now derive an algorithm that finds the exact same final predictions with much more congenial computation, by carefully analyzing the total change to the weight vector due to the LMS algorithm.

The final step of the algorithm in \eqref{FV_tau} can be rewritten as
\begin{align}
\th_{\tau}
& = \th_{\tau-1} + \a_{\tau-1} \p_{\tau-1} \left( \Z - \p_{\tau-1}\tr\th_{\tau-1} \right)  \nonumber\\
& = \th_{\tau-1} + \a_{\tau-1}\p_{\tau-1}\Z - \a_{\tau-1} \p_{\tau-1} \p_{\tau-1}\tr \th_{\tau-1}\nonumber\\
& = \left( \I - \a_{\tau-1} \p_{\tau-1}  \p_{\tau-1} \tr \right) \th_{\tau-1}  + \a_{\tau-1} \p_{\tau-1} \Z  \nonumber\\ 
& = \F_{\tau-1}  \th_{\tau-1}  + \a_{\tau-1} \p_{\tau-1} \Z \,.\nonumber
\end{align}
Here $\F_t \doteq \I - \a_t \p_t \p_t\tr$ is a \emph{fading} matrix that will be important throughout this paper. Now, continuing,
\begin{align}
\th_{\tau}
& = \F_{\tau-1}  \left( \F_{\tau-2} \th_{\tau-2} + \a_{\tau-2} \p_{\tau-2} \Z \right) + \a_{\tau-1} \p_{\tau-1} \Z\tag{expanding $\th_{\tau-1}$}\\
& = \F_{\tau-1} \F_{\tau-2} \th_{\tau-2} + \left( \F_{\tau-1} \a_{\tau-2}\p_{\tau-2} +  \a_{\tau-1} \p_{\tau-1}  \right) \Z \tag{regrouping}\\
& = \F_{\tau-1} \F_{\tau-2} \left( \F_{\tau-3} \th_{\tau-3} + \a_{\tau-3} \p_{\tau-3} \Z \right) + \left( \F_{\tau-1} \a_{\tau-2}\p_{\tau-2} +  \a_{\tau-1} \p_{\tau-1}  \right) \Z \tag{recursing on $\th_{\tau-2}$}\\
& = \F_{\tau-1} \F_{\tau-2} \F_{\tau-3}\th_{\tau-3} + \left(\F_{\tau-1} \F_{\tau-2}\a_{\tau-3}\p_{\tau-3} + \F_{\tau-1} \a_{\tau-2}\p_{\tau-2} + \a_{\tau-1}\p_{\tau-1}  \right) \Z \tag{regrouping} \\
&~~\vdots\tag{recursing further}\\
& = \underbrace{\F_{\tau-1} \F_{\tau-2}\cdots \F_0\th_0}_{\mbox{$\doteq\aa_{\tau-1}$}} ~+~~  \underbrace{\left( \sum_{t=0}^{\tau-1}\F_{\tau-1} \F_{\tau-2}  \cdots \F_{t+1} \a_t\p_t\right) }_{\mbox{$\doteq\ee_{\tau-1}$}} \Z \nonumber\\
& = \aa_{\tau-1}  + \ee_{\tau-1} \Z \, , \label{th_tau}
\end{align}
where $\aa_t$ and $\ee_t$ are two auxiliary memory vectors.

Importantly, the auxiliary vectors can be updated without knowledge of $\Z$, and with complexity independent of span and proportional to the number of features.
The $\aa_t$ vector stores the effect of the initial weights on the updated weights.  It is initialized as $\aa_0=\th_0$ and can then be updated efficiently with
\begin{align}
\aa_t 
& \doteq \F_t \F_{t-1} \cdots \F_0\th_0 \notag\\
& = \F_t \left( \F_{t-1} \cdots \F_1\th_1 \right) \notag\\
& = \F_t \aa_{t-1}  \notag\\
& =  \aa_{t-1} - \a_t \p_t \p_t\tr \aa_{t-1} \notag\\
& =  \aa_{t-1} + \a_t \p_t
 ( 0 - \p_t\tr \aa_{t-1} ) \,, \qquad t=1, \ldots, \tau-1. \label{aa}
\end{align}
The $\ee_t$ vector is analogous to the conventional eligibility trace \parencite[see:][and references therein]{Sutton:1988,SuttonBarto:1998} but has a special form as first proposed by \textcite{vanSeijen:2014}. It is initialized to $\ee_{-1}=\bm 0$ (or, equivalently, to $\ee_0 = \a_0 \p_0$) and then updated according to
\begin{align} 
\ee_t
& \doteq \sum_{k=0}^t \F_t \F_{t-1} \cdots \F_{k+1} \a_k\p_k \nonumber\\
& = \sum_{k=0}^{t-1} \F_t \F_{t-1} \cdots \F_{k+1} \a_k\p_k ~~+~~ \a_t\p_t \nonumber\\
& = \F_t \underbrace{\sum_{k=0}^{t-1} \F_{t-1} \F_{t-2} \cdots \F_{k+1} \a_k\p_k}_{\mbox{$= \ee_{t-1}$}} ~~+~~ \a_t\p_t \nonumber\\
& = \F_t \ee_{t-1} + \a_t\p_t \label{Fee}\\
& = \ee_{t-1} - \a_t \p_t \p_t\tr \ee_{t-1} + \a_t\p_t \nonumber\\
& = \ee_{t-1} + \a_t \p_t ( 1 - \p_t\tr \ee_{t-1}) 
, \qquad t=0, \ldots, \tau-1.\label{trace} 
\end{align}
An eligibility trace of this special form is called a \emph{dutch trace} \parencite{vanHasselt:2014}.  For comparison, the conventional accumulating trace that it replaces can be written as $\ee_{-1} \doteq {\bm 0}$ and $\ee_t \doteq \ee_{t-1} + \a_t \p_t$.\footnote{We incorporate the step size into both trace updates.  This is a slight deviation from the way these traces are usually written to allow for time-changing step sizes and increased generality.}

The emergence of the dutch trace here is surprising and intriguing because, in contrast to previous work \parencite{vanSeijen:2014,vanHasselt:2014}, the dutch trace has arisen in a setting without temporal-difference (TD) learning.  Eligibility traces are not specific to TD learning at all; they are more fundamental than that.  The need for eligibility traces seems to arise whenever one tries to learn long-term predictions in an efficient manner, that is, with computational complexity that is independent of predictive span.

The auxiliary vectors $\aa_t$ and $\ee_t$ are updated on each time step $t<\tau$ and then, after observing $\Z$ at time $\tau$, are used to compute $\th_{\tau} = \aa_{\tau-1} + \ee_{\tau-1} \Z$, as in \eqref{th_tau}. This way we achieve exactly the same final result as the forward view \eqref{FV_tau}, but with an algorithm whose time and memory complexity per step is uniformly $O(n)$ and independent of span. The complete algorithm can be summarized as:
\beq \label{algFinal}
\begin{array}{ll}
\aa_0 \doteq  \th_0\,,\; \text{ then }\aa_{t+1} \doteq  \aa_{t} + \a_t\p_t ( 0 - \p_t\tr\aa_{t} ),  & t=1, \ldots, \tau-1\,, \\
\ee_{-1} \doteq  {\bm 0}\,,\, \text{ then }\ee_t \doteq  \ee_{t-1} + \a_t\p_t ( 1 - \p_t\tr\ee_{t-1} ),  & t=0, \ldots, \tau-1\,,\\
\th_{\tau} \doteq  \aa_{\tau-1} + Z\ee_{\tau-1}\,.
\end{array}
\eeq
The vector $\aa_{\tau-1}$ can be interpreted as storing the remaining effect of the initial weights $\th_0$ after all updates in \eqref{FV_tau} have concluded.  The trace $\ee_{\tau-1}$ can be interpreted as storing all we need to know about the feature vectors that were observed during the episode. Together, these vectors allow us to replace all the $\tau$ updates of the forward view with one fully equivalent update at the end of the episode.

We call the span-independent algorithm \eqref{algFinal} the \emph{backward view} corresponding to the forward view defined in \eqref{FV_tau}, because on each step all updates only use information that is available at that time step: we only look backwards in time.  The advantage of this is that the updates can be computed immediately and we do not have to wait and store observations until later.  Until recently, exactly equivalences between forward and backward views were only known to exist for algorithms that update their predictions in batch \citep{SuttonBarto:1998}.  Van Seijen \& Sutton (2014) were the first to derive an online backward view that was exactly equivalent to its forward view in terms of the learned predictions.  The derivation in this section shows that such equivalences exist more generally, including for the LMS update in \eqref{FV_tau}.

When we consider the episode as a whole, there is no gain in total computation time for the span-independent algorithm \eqref{algFinal} compared to the conventional algorithm \eqref{FV_tau}: both algorithms use $O(n\tau)$ computation for the entire episode.  However, in the span-independent algorithm the computation is spread out more evenly with a uniform per-step complexity of $O(n)$, whereas the conventional algorithm performs the bulk of computation at the end, when we finally observe $\Z$.
\begin{figure}
\begin{center}
\includegraphics[width=5.4in]{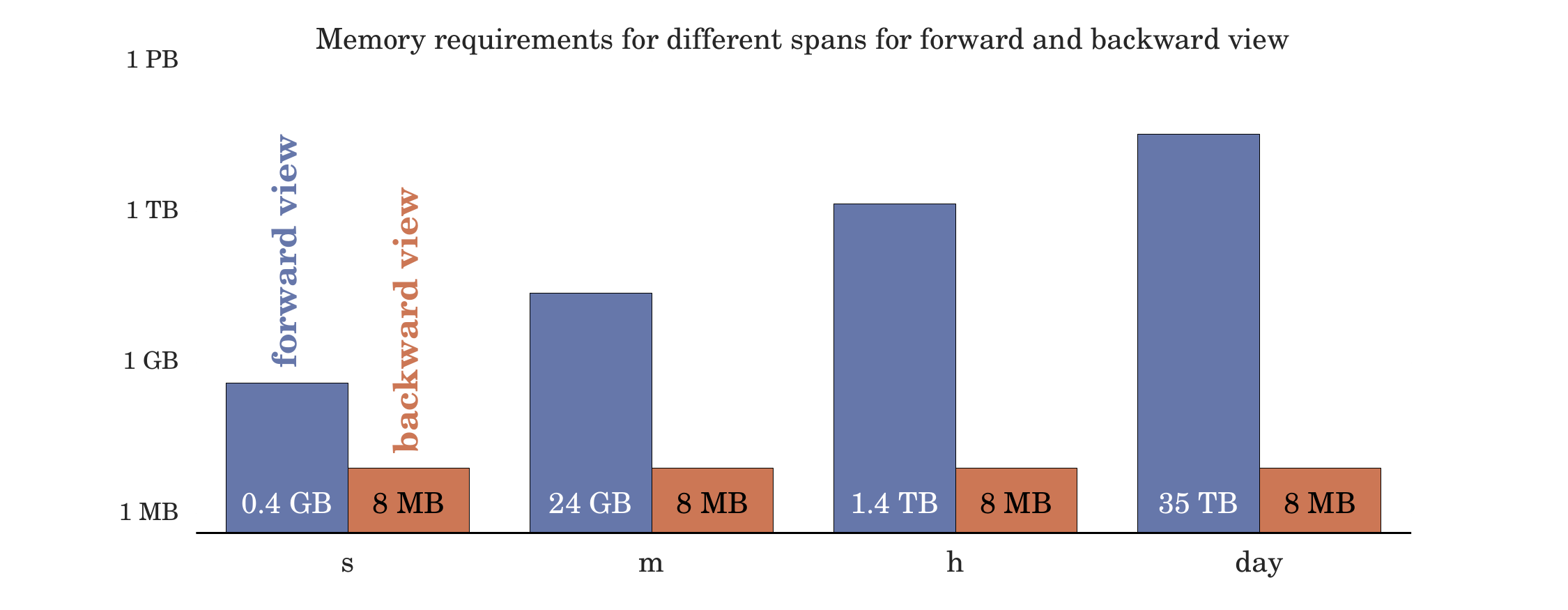}
\caption{\textbf{Example memory requirements.} \textit{The bars show memory requirements for the conventional algorithm \eqref{FV_tau} and the span-independent algorithm \eqref{algFinal}. The brown bars of the conventional algorithm \eqref{FV_tau} that stores all observations over the duration of a second, minute, hour, or day, when one million features of 4 bytes each are observed every $10~\text{ms}$. The blue bars on the right show the memory required by the span-independent algorithm \eqref{algFinal} in the same setting, which is 8~MB for any span. \label{fig:robot_memory}}}
\end{center}
\end{figure}

Additionally, there is a gain in terms of required memory.  Algorithm \eqref{FV_tau} needs to store all previously observed features, leading to memory requirements of order $O(n\tau)$, whereas algorithm \eqref{algFinal} only needs to store $\aa$ and $\ee$ and therefore has constant span-independent memory requirements of order $O(n)$.
In real-world problems, for instance in robotics, it is not uncommon to extract millions of features from the sensory inputs at each step \parencite[e.g.,][]{Montemerlo:2003}, where each step lasts only a fraction of a second.  Consider a robot that generates one million features each $10~\text{ms}$, where each feature is a real number represented with single precision using 4 bytes of memory.  Figure \ref{fig:robot_memory} shows the resulting memory requirements for the conventional algorithm and the span-independent algorithm for different spans of the predictions.  The required storage of the conventional algorithm ranges from 0.4 GB for predictions spanning one second to about 35 Terabytes for predictions spanning one day. While a few Gigabytes of on-board storage is feasible with today's resources, several Terabits will be a significant burden for an autonomous mobile robot.  Concretely, this means that with the conventional algorithm we would have to restrict either the number of features, the frequency at which we make predictions, or the maximum amount of wall time a prediction can span.
The span-independent algorithm scales much better.  In the example in Figure \ref{fig:robot_memory} it only needs to store two vectors with $10^6$ components of 4 bytes each, resulting in memory requirements that are constant at 8~Megabytes.

\section{Online updating and the emergence of TD errors}\label{sec:online}


The algorithms in the previous section do not make any changes to the predictions during the episode; these are \emph{offline} algorithms. 
Yet it would sometimes make sense to update the predictions during an episode, especially if the span is very long and we do not want to wait that long before we start learning.  In this section we introduce an online forward view and derive a span-independent algorithm (the backward view) that on each time step computes the exact same predictions.

An online algorithm cannot update the predictions towards the final outcome $\Z$ during the episode, because $\Z$ is not yet available.  Instead, if we only have observations up to a horizon $h<\tau$ we may want to move the predictions for all earlier times $t<h$ towards some informed guess of what the final outcome will be.  Such a guess plays the role of a target for the updates, like $\Z$ in the forward view \eqref{FV_tau}, but it is used prior to $\Z$ being available; it is an \emph{interim target}.  We use $Z^h$ to denote the interim target at time $h$, which may be based on all the data available up that horizon. The interim target might be from a human expert or it might be, as in TD learning, the current prediction corresponding to the feature vector $\p_h$. Interim targets at times closer to $\tau$ might produce more accurate predictions. In the example of the stock market, as we get closer to the end of the year we may be able to more accurately estimate the final stock price. For now we consider the general case and do not specify the source of the interim targets $Z^h$, for $h=1, \ldots, \tau-1$. Notationally it is convenient to define $\Z^{\tau}\doteq\Z$.
Note that the time index on $\Z^h$ is a superscript rather than a subscript. Our convention is that the superscript position is reserved for the upper limit of the data considered available in an online update. The subscript position is used for the time step whose prediction is being modified. 

To clarify the notation under online updating, we introduce the notation $\t{t}{h}$ for the weights at step $t$ based on all the data up through time $h$.  Using these double subscripts, what we previously called $\th_t$ would now be $\t{t}{\tau}$, because these weight vectors depend on $\Z$ which is considered to arrive at time $\tau$.
The complete set of online updates is then
\begin{equation}\label{interim_FV}
\begin{array}{ll}
\t{0}{h} \doteq \th_0,  &h=0, \ldots, \tau\,;\\
\t{t+1}{h} \doteq \t{t}{h} + \a_t \p_t \left( \Z^h - \p_t\tr\t{t}{h} \right) , 
~~&t=0, \ldots, h-1\,, ~~ h=1, \ldots, \tau\,; \\
\phantom{\t{t+1}{h}}
= \F_t\t{t}{h} + \a_t\p_t\Z^h.
\end{array}
\end{equation}
This online algorithm defines a set of $h$ updates for each interim horizon $h \in \{1, \ldots, \tau \}$.  For each horizon, all predictions are updated towards the latest, and presumably best available, interim target.  Although most updates do not involve the final outcome $\Z$, this algorithm is still considered a forward view, because the prediction at some time $t$ (e.g., $\p_t\tr\t{t}{h}$) is updated using interim targets that arrive later (e.g., $\Z^h$).

We can write out all the double-subscripted weight vectors in a triangle as
\begin{equation}\label{eq:triangle}
\begin{array}{llllllll}
\t{0}{0} & & & & & & \\
\t{0}{1} & \t{1}{1} & & & & & & \\
~\vdots & ~\vdots & \ddots & & & & &  \\
\t{0}{h} & \t{1}{h} & \ldots & \t{h}{h} &&  \\
\t{0}{h+1} & \t{1}{h+1} & \ldots & \t{h}{h+1} & \t{h+1}{h+1} &&  \\
~\vdots & ~\vdots &  & \vdots & \vdots &\ddots & &   \\
\t{0}{\tau} & \t{1}{\tau} & \ldots & \t{h}{\tau} & \t{h+1}{\tau} & \ldots & \t{\tau}{\tau} \,.
\end{array}
\end{equation}
The computation proceeds from top to bottom, one row at a time from left to right. Each row starts with the same initial values $\t{0}{h} = \th_0$ and then computes the sequence $\t{1}{h}$, $\t{2}{h}$ and so on  If we really computed all the different weight vectors in the triangle then the algorithm would be inefficient.  With a computational complexity of $O(nh)$ on each time $h \in \{ 1, \ldots, \tau \}$, the computation would not be constant per time step and would not be independent of the span of the prediction; the last row alone has complexity $O(n\tau)$ and it can only be computed after observing $\Z$ at time $\tau$.  However, instead of computing the whole triangle, perhaps there is a way to incrementally compute just the diagonal, to somehow obtain $\t{h+1}{h+1}$ from $\t{h}{h}$ efficiently on each step. If this can be done, then the entire computation will be of uniform $O(n)$ complexity per time step, independent of span.

To find an efficient update along the diagonal of the triangle, notice first that the forward view \eqref{interim_FV} already provides a way to efficiently move right one step in the triangle.  In other words, we can get to $\t{h+1}{h+1}$ from $\t{h}{h+1}$ for any $h$ with constant $O(n)$ computation.
If we can find an efficient way to step down in the triangle, that is to get to $\t{h}{h+1}$ from $\t{h}{h}$ for any $h$ with constant $O(n)$ computation, then we can combine these two steps into a single $O(n)$ update.  To see if this is possible, we first write down explicitly how each weight vector in the triangle depends on the initial weights and the interim targets.  Similar to our derivation of the final weights in \eqref{th_tau} in the previous section, we can unroll the forward-view updates repeatedly starting from $\t{t}{h}$ and obtaining
\begin{align}
\t{t}{h}
& = \F_{t-1} \t{t-1}{h} + \a_{t-1} \p_{t-1}\Z^{h}  \tag{applying \eqref{interim_FV} to $\t{t}{h}$}\\
& = \F_{t-1} ( \F_{t-2} \t{t-2}{h} +  \a_{t-2}\p_{t-2}\Z^{h}) + \a_{t-1} \p_{t-1} \Z^{h} \tag{applying \eqref{interim_FV} to $\t{t-1}{h}$}\\
& = \F_{t-1} \F_{t-2}\t{t-2}{h} + \left( \F_{t-1} \a_{t-2} \p_{t-2} + \a_{t-1} \p_{t-1} \right) \Z^{h} \tag{regrouping}\\
& = \F_{t-1} \F_{t-2}(\F_{t-3} \t{t-3}{h} +  \a_{t-3} \p_{t-3}\Z^{h} ) + \left( \F_{t-1} \a_{t-2} \p_{t-2} + \a_{t-1} \p_{t-1} \right)  \Z^{h}\tag{applying \eqref{interim_FV} to $\t{t-2}{h}$}\\
& = \F_{t-1} \F_{t-2}\F_{t-3} \t{t-3}{h} +  \left( \F_{t-1}\F_{t-2} \a_{t-3}\p_{t-3} + \F_{t-1} \a_{t-2} \p_{t-2} + \a_{t-1} \p_{t-1} \right)\Z^{h} \notag\\
& ~~\vdots \tag{continuing until we reach $\t{0}{h}$}\\
& = \underbrace{\F_{t-1} \cdots \F_0 \th_0}_{\mbox{$\aa_{t-1}$}} ~~+~~ \underbrace{\left(\sum_{j=0}^{t-1} \F_{t-1} \cdots \F_{j+1} \a_j \p_j\right)}_{\mbox{$\ee_{t-1}$}}\Z^h \notag\\
& = \aa_{t-1} + \ee_{t-1}\Z^h  \,. \label{th_th}
\end{align}
Notice that $\aa_{t-1}$ and $\ee_{t-1}$ depend only on time $t$ and not on the data horizon $h$.  We can use this result to  find the difference between $\t{h}{h+1}$ and $\t{h}{h}$ as
\begin{align}
\t{h}{h+1} - \t{h}{h} 
& = (\aa_{h-1} + \ee_{h-1} \Z^{h+1}) - (\aa_{h-1} + \ee_{h-1} \Z^h) \notag\\
& = \ee_{h-1} ( \Z^{h+1} - \Z^h ) \,.\label{step2}
\end{align}
Here we see the emergence of a temporal difference error $\Z^{h+1} - \Z^h$.  We know from the previous section that $\ee_{h-1}$ can be computed incrementally with \eqref{trace} and therefore does not have to be recomputed from scratch for each new observation.  Therefore, we now have an efficient way to compute $\t{h}{h+1}$ from $\t{h}{h}$ for any $h$ with constant $O(n)$ computation per step that is independent of span.  Now that we have $\t{h}{h+1}$ we can efficiently compute $\t{h+1}{h+1}$ using \eqref{interim_FV}, and we can merge these two steps to compute $\t{h+1}{h+1}$ directly from $\t{h}{h}$. The complete update can then be written as
\begin{align*}
\t{h+1}{h+1}
& = \F_h \t{h}{h+1} + \a_h \p_h \Z^{h+1} \tag{using \eqref{interim_FV}} \\ 
& =  \F_h \left( \t{h}{h} + \ee_{h-1} (\Z^{h+1} - \Z^h ) \right) + \a_h \p_h \Z^{h+1} \tag{using \eqref{step2}} \\
& =  \F_h \t{h}{h} + \F_h \ee_{h-1} (\Z^{h+1} - \Z^h )  + \a_h \p_h \Z^{h+1} \\
& =  \F_h \t{h}{h} + ( \ee_h - \a_h \p_h )(\Z^{h+1} - \Z^h )  + \a_h \p_h \Z^{h+1} \tag{using \eqref{Fee}} \\
& =  \F_h \t{h}{h} +  \ee_h (\Z^{h+1} - \Z^h )- \a_h \p_h (\Z^{h+1} - \Z^h ) + \a_h \p_h \Z^{h+1} \\\
& =  \F_h \t{h}{h} +  \ee_h (\Z^{h+1} - \Z^h )+ \a_h \p_h \Z^h \\
& =  (\I - \a_h \p_h \p_h\tr ) \t{h}{h} + \ee_h(\Z^{h+1} - \Z^h ) + \a_h \p_h \Z^h \tag{by definition of $\F_h$}\\
& = \t{h}{h} + \ee_h (\Z^{h+1} - \Z^h ) + \a_h \p_h \left( \Z^h - \p_h \tr \t{h}{h} \right) \,.
\end{align*}
This update holds for all $h\ge 1$.  The update for $\t{1}{1}$ is given directly by \eqref{interim_FV} as
\[
\t{1}{1} = \t{0}{1} + \a_0 \p_0 (\Z^1-\p_0\tr\t{0}{1}) \,.
\]
For any $\Z^0$ we can rewrite this as
\begin{align*}
\t{1}{1}
&= \t{0}{0} + \a_0 \p_0 (\Z^1-\p_0\tr\t{0}{0}) \tag{using $\t{0}{1} \doteq \th_0 \doteq \t{0}{0}$}\\
&= \t{0}{0} + \a_0 \p_0 (\Z^1- \Z^0 + \Z^0 - \p_0\tr\t{0}{0}) \\
&= \t{0}{0} + \a_0 \p_0 (\Z^1- \Z^0 ) + \a_0 \p_0 (\Z^0 - \p_0\tr\t{0}{0}) \\
&= \t{0}{0} + \ee_0 (\Z^1- \Z^0) + \a_0 \p_0 (\Z^0 - \p_0\tr\t{0}{0}) \,,\tag{using $\ee_0 \doteq \a_0 \p_0$}
\end{align*}
which means that then the update derived above for $h \ge 1$ in fact also holds for $h = 0$.  For concreteness, we will define $\Z^0 \doteq 0$, even though this value does not actually affect any of the weights.

Now that we have an update that can compute $\t{t+1}{t+1}$ from $\t{t}{t}$ for any $t$, we can drop the redundant superscript.  The resulting algorithm is
\beq \label{algTDFinal}
\begin{array}{ll}
\ee_{-1} \doteq {\bm 0}, \text{ then }\ee_t \doteq \ee_{t-1} + \a_t \p_t ( 1 - \p_t\tr\ee_{t-1} ),  & t=0, \ldots, \tau-1, \\
\th_{t+1} \doteq  \th_t + \ee_t\left(\Z^{t+1}-\Z^t\right) +\a_t \p_t (\Z^t-\p_t\tr\th_{t}) , & t=0, \ldots, \tau-1.
\end{array}
\eeq
By construction this backward view is equivalent to the less efficient forward view \eqref{interim_FV} in the sense that $\th_t = \t{t}{t}$ for all $t$.  In contrast to the offline backward view \eqref{algFinal} that we derived in the previous section, we no longer need to compute and store the auxiliary vector $\aa_t$.  All relevant information that was contained therein is now stored directly in the online weights $\th_t$.
%

Although the online backward view yields different predictions during the episode, the final weights $\th_{\tau}$ are exactly equal to those computed by the conventional LMS algorithm \eqref{FV_tau} that constituted our first, offline, forward view.  In terms of the triangle in \eqref{eq:triangle}, the online forward view \eqref{interim_FV} computes the whole triangle, the online backward view \eqref{algTDFinal} efficiently computes only the diagonal, the offline forward view \eqref{FV_tau} computes only the last row, and the offline backward view \eqref{algFinal} from the previous section computes only the final weights.  All three algorithms ultimately result in the same final weights.

\section{Unifying online and offline learning and the emergence of averaging}\label{sec:trust}
The online algorithms from the previous section do not quite subsume the offline algorithms from Section \ref{sec:span}. Although they all reach the same weights by the end of the episode, during the episode their weights are different. The offline algorithm does not change the weights during the episode, and the online algorithm must change them.

One might think that the online algorithm is always better because it can immediately use any incoming relevant information, but it is not so. Suppose the interim targets are always wildly wrong (say due to a poor human `expert'). They would cause the weights of the online algorithm to also be wildly wrong for all steps except the last one at the end of the episode. In this case the weights of the online algorithm would be worse than those of the offline algorithm almost all of the time.

Because interim targets can sometimes be misleading, we might want to reduce their effect on some steps, based on how much we trust these targets.
In this section $\onth_t$ denotes the online weights, which are computed by the span-independent backward view \eqref{algTDFinal}.  The unified weights $\tth_t$ take into account the degree of trust for each interim target and are used to make our predictions. 
If we trust $\Z^t$ fully, we want to obtain the same weights as in the online algorithm, so that $\tth_t = \onth_t$.  If we do not trust $\Z^t$ at all, we want the predictions to remain unchanged, so that $\tth_t = \tth_{t-1}$.  For intermediate degrees of trust, $\b_t \in (0,1)$, the algorithm should smoothly move from one extreme to the other, so the final result of the update should be something like
\beq\label{th_trust}
\tth_{t+1} = (1 - \b_{t+1})\tth_t + \b_{t+1} \onth_{t+1} \,.
\eeq
The above reasoning may sound plausible, but is it sound?  In this section we construct a forward view for partially trusted interim targets, and then derive an equivalent span-independent backward view.  It turns out the resulting update indeed changes the weights precisely as in \eqref{th_trust}.

In the forward view, the online weights that always trust the latest interim target fully will be denoted $\ont{t}{h}$ to differentiate them from the trusted interim weights $\tt{t}{h}$.  
If we have data up to horizon $h$ but we trust the latest interim target $\Z^h$ only with degree $\b_h \in [0,1]$, then the predictions prior to $h$ should update towards $\Z^h$ only to this degree and for the rest, with degree $(1 - \b_h)$, fall back on earlier interim targets.  Similarly we update towards $\Z^{h-1}$ only as far as this interim target was trusted, so with total degree $(1 - \b_h) \b_{h-1}$,   and then further fall back to $\Z^{h-2}$ with total degree $(1 - \b_h) ( 1 - \b_{h-1} ) \b_{h-2}$, and so on.  If we do not trust any of the interim targets observed between the time $t$ of making the prediction and the current horizon $h$, the predictions should remain wherever they were at time $t$.  This can be achieved by updating the prediction at time $t$ towards the then-current prediction $\p_t\tr\tt{t}{t}$ with the remaining degree $(1 - \b_h) (1 - \b_{h-1})\cdots (1 - \b_{t+1})$.  Note that the different multipliers sum to one:
\[
\b_h ~+~ (1 - \b_h) \b_{h-1} ~+~ (1 - \b_h) ( 1 - \b_{h-1} ) \b_{h-2} ~+~ \ldots ~+~ (1 - \b_h) \cdots (1 - \b_{t+1}) = 1 \,.
\]
The total forward view for a prediction at time $t$ with a horizon of $h$ is therefore given by
\def\rl#1#2{\begin{array}{rl}#1 & #2 \end{array}}
\begin{align}
\tt{t+1}{h}
& = \tt{t}{h}   + \b_h \a_t \p_t ( \Z^h- \p_{t}\tr\tt{t}{h} ) \notag\\
& \hspace{29pt} + (1 - \b_h) \b_{h-1} \a_t \p_t ( \Z^{h-1} - \p_{t}\tr\tt{t}{h} ) \notag\\
& \hspace{29pt} + (1 - \b_h) ( 1- \b_{h-1}) \b_{h-2} \a_t \p_t ( \Z^{h-2} - \p_{t}\tr\tt{t}{h} ) \notag\\
& \hspace{29pt}  ~~ \vdots \notag\\
& \hspace{29pt}  + (1 - \b_h)  \cdots (1 - \b_{t+2}) \b_{t+1} \a_t \p_t ( \Z^{t+1} - \p_{t}\tr\tt{t}{h} ) \notag\\
& \hspace{29pt}  + (1 - \b_h) \cdots (1 - \b_{t+2}) (1 - \b_{t+1}) \a_t \p_t ( \p_t\tr\tt{t}{t} - \p_{t}\tr\tt{t}{h} ) \,.\notag \\
\noalign{(now grouping terms $\a_t\p_t( \cdot -  \p_t\tr\tt{t}{h})$ with total weight equal to one)}
& = \tt{t}{h} + \a_t \p_t \bigg( \b_h \Z^h + \notag\\
& \hspace{67pt} (1 - \b_h) \b_{h-1}\Z^{h-1} + \notag\\
& \hspace{67pt} (1 - \b_h) ( 1- \b_{h-1}) \b_{h-2}\Z^{h-2} + \notag\\
& \hspace{67pt}  ~~ \vdots \notag\\
& \hspace{67pt}  (1 - \b_h)  \cdots (1 - \b_{t+2}) \b_{t+1} \Z^{t+1} + \notag\\
& \hspace{67pt}  (1 - \b_h) \cdots (1 - \b_{t+2}) (1 - \b_{t+1}) \p_t\tr\tt{t}{t} - \p_{t}\tr\tt{t}{h} \bigg) \notag\\
& = \tt{t}{h} + \a_t \p_t ( \Zb{t}{h} - \p_t\tr\tt{t}{h} ) \notag\\
& = \F_t \tt{t}{h} + \a_t \p_t \Zb{t}{h} \,, \qquad\qquad t=0, \ldots, h-1\,; ~~~ h=1, \ldots, \tau, \label{FV_trust} \\
\text{where}\quad\Zb{t}{h}
& \doteq \b_h \Z^h \notag\\
& \quad + (1 - \b_h) \b_{h-1} \Z^{h-1} \notag\\
& \quad + (1 - \b_h) ( 1- \b_{h-1}) \b_{h-2}\Z^{h-2} \notag\\
& \quad + \ldots \notag\\
& \quad + (1 - \b_h) \cdots (1 - \b_{t+2}) \b_{t+1} \Z^{t+1} \notag\\
& \quad + (1 - \b_h) \cdots (1 - \b_{t+2}) (1  -\b_{t+1} ) \p_t\tr\tth_t \notag\\
& = \b_h \Z^h + (1 - \b_h) \Zb{t}{h-1} \,,\qquad t=0, \ldots, h-1\,; ~~ h=1, \ldots, \tau ~\text{, and}\label{Zb_def}\\
\Zb{t}{t} & \doteq \p_t\tr\tt{t}{t} \hspace{3.45cm} t = 0, \ldots, \tau-1\,. \label{Zb_tt} 
\end{align}

To derive a span-independent variant of algorithm \eqref{FV_trust}, we first identify how a general weight vector $\tt{t}{h}$ depends on the combined interim targets and the observed feature vectors by applying the recursive definition in \eqref{FV_trust} repeatedly, yielding
\begin{align}
\tt{t}{h} 
& = \F_{t-1} \tt{t-1}{h} + \a_{t-1} \p_{t-1} \Zb{t-1}{h} \tag{applying \eqref{FV_trust} to $\tt{t}{h}$}\\
& = \F_{t-1} \left( \F_{t-2} \tt{t-2}{h} + \a_{t-2} \p_{t-2} \Zb{t-2}{h} \right) + \a_{t-1} \p_{t-1} \Zb{t-1}{h} \tag{applying \eqref{FV_trust} to $\tt{t-1}{h}$}\\
& = \F_{t-1} \F_{t-2} \tt{t-2}{h} + \F_{t-1} \a_{t-2} \p_{t-2} \Zb{t-2}{h} + \a_{t-1} \p_{t-1} \Zb{t-1}{h} \notag\\
& = \F_{t-1} \F_{t-2} ( \F_{t-3} \tt{t-3}{h} + \a_{t-3} \p_{t-3} \Zb{t-3}{h} ) + \F_{t-1} \a_{t-2} \p_{t-2} \Zb{t-2}{h} + \a_{t-1} \p_{t-1} \Zb{t-1}{h} \tag{applying \eqref{FV_trust} to $\tt{t-2}{h}$}\\
& = \F_{t-1} \F_{t-2} \F_{t-3} \tt{t-3}{h} + \F_{t-1} \F_{t-2} \a_{t-3} \p_{t-3} \Zb{t-3}{h} + \F_{t-1} \a_{t-2} \p_{t-2} \Zb{t-2}{h} + \a_{t-1} \p_{t-1} \Zb{t-1}{h} \notag\\
& ~~\vdots\notag\\
& = \underbrace{\F_{t-1} \cdots \F_0 \tt{0}{t}}_{\mbox{$\doteq \aa_{t-1}$}} ~+~ \sum_{k=0}^{t-1} \F_{t-1} \cdots \F_{k+1} \a_k \p_k \Zb{k}{h} \notag\\
& = \aa_{t-1} ~+~ \sum_{k=0}^{t-1} \F_{t-1} \cdots \F_{k+1} \a_k \p_k \Zb{k}{h} \label{th_th_beta} \,.
\end{align}
The last step uses the fact that the initial trusted weights $\tt{0}{t}$ are equal to the initial online weights, such that $\tt{0}{t} = \ont{0}{t} = \th_0$ for any $t$, which means $\aa_t$ is the same as before.
Notice that we have not yet used the definition of $\Zb{t}{h}$ in any way; the derivation so far holds for any combined target.

We now first examine if we can efficiently go down in the triangle, that is, to get to $\tt{h}{h+1}$ from $\tt{h}{h}$:
\begin{align}
\tt{h}{h+1} - \tt{h}{h}
& = \left( \aa_{h-1} + \sum_{k=0}^{h-1} \F_{h-1} \cdots \F_{k+1} \a_k \p_k \Zb{k}{h+1} \right) \tag{$\tt{h}{h+1}$ from \eqref{th_th_beta}}\\
& \qquad - \left( \aa_{h-1} + \sum_{k=0}^{h-1} \F_{h-1} \cdots \F_{k+1} \a_k \p_k \Zb{k}{h} \right) \tag{$\tt{h}{h}$ from \eqref{th_th_beta}}\\
& = \sum_{k=0}^{h-1} \F_{h-1} \cdots \F_{k+1} \a_k \p_k \big( \Zb{k}{h+1} \!- \Zb{k}{h} \big) \tag{merge sums, cancel $\aa_{h-1}$}\\
& = \sum_{k=0}^{h-1} \F_{h-1} \cdots \F_{k+1} \a_k \p_k \bigg( \b_{h+1} \Z^{h+1} + (1 - \b_{h+1}) \Zb{k}{h} - \Zb{k}{h} \bigg) \tag{using \eqref{Zb_def} on $\Zb{k}{h+1}$}\\
& = \sum_{k=0}^{h-1} \F_{h-1} \cdots \F_{k+1} \a_k \p_k \b_{h+1} ( \Z^{h+1} -  \Zb{k}{h} ) \notag\\
& =  \b_{h+1} \underbrace{\left( \sum_{k=0}^{h-1} \F_{h-1} \cdots \F_{k+1} \a_k \p_k \right)}_{\mbox{$=\ee_{h-1}$}} \Z^{h+1} ~~-~~ \b_{h+1} \underbrace{\sum_{k=0}^{h-1} \F_{h-1} \cdots \F_{k+1} \a_k \p_k \Zb{k}{h}}_{\mbox{$= \tt{h}{h}  - \aa_{h-1}$, from \eqref{th_th_beta}}} \notag\\
& =  \b_{h+1} \ee_{h-1} \Z^{h+1} ~~-~~ \b_{h+1} (\tt{h}{h}  - \aa_{h-1}) \notag\\
& =  \b_{h+1} \underbrace{( \aa_{h-1} + \ee_{h-1} \Z^{h+1} )}_{\mbox{$= \ont{h}{h+1}$, from \eqref{th_th}}} ~~-~~ \b_{h+1} \tt{h}{h} \notag\\
& = \b_{h+1} ( \ont{h}{h+1} - \tt{h}{h} )
\end{align}
Thus $\tt{h}{h+1}$ can be written as a simple combination of the previous trusted weights $\tt{h}{h}$ and the interim online weights $\ont{h}{h+1}$ with
\beq\label{step2_trust}
\tt{h}{h+1} = ( 1 - \b_{h+1} ) \tt{h}{h} + \b_{h+1} \ont{h}{h+1} \,.
\eeq
We can plug this value into the definition of $\tt{h+1}{h+1}$ to find
\begin{align*}
\tt{h+1}{h+1}
& = \F_h \tt{h}{h+1} + \a_h \Zb{h}{h+1} \p_h &\tag{using \eqref{FV_trust}} \\
& = \F_h \bigg( ( 1- \b_{h+1} ) \tt{h}{h} + \b_{h+1} \ont{h}{h+1}  \bigg) + \a_h \Zb{h}{h+1} \p_h &\tag{using \eqref{step2_trust}}\\
& = \F_h \bigg( ( 1- \b_{h+1} ) \tt{h}{h} + \b_{h+1} \ont{h}{h+1}  \bigg) + \a_h \left( \b_{h+1} \Z^{h+1} + (1 - \b_{h+1} ) \Z^h \right) \p_h \,. &\tag{using \eqref{Zb_def}} \\
& = \F_h \bigg( ( 1- \b_{h+1} ) \tt{h}{h} + \b_{h+1} \ont{h}{h+1}  \bigg) + \a_h \left( \b_{h+1} \Z^{h+1} + (1 - \b_{h+1} ) \p_h\tr\tt{h}{h} \right) \p_h \,. &\tag{using \eqref{Zb_tt}}
\end{align*}
Now we group the terms depending on whether they are trusted (multiplied with $\b_{h+1}$) or untrusted (multiplied with $(1 - \b_{h+1})$) to simply further to
\begin{align*}
\tt{h+1}{h+1}
& = ( 1 - \b_{h+1} ) \underbrace{\bigg( \F_h  \tt{h}{h} + \a_h \p_h \p_h \tr \tt{h}{h} \bigg)}_{=\tt{h}{h} \text{, using } \F_h = \I - \a_h \p_h \p_h\tr} ~~+~~ \b_{h+1} \underbrace{\bigg( \F_h \ont{h}{h+1} + \a_h \Z^{h+1} \p_h \bigg)}_{\mbox{$=\ont{h+1}{h+1}$, using \eqref{interim_FV}}} \\
& = ( 1 - \b_{h+1} ) \tt{h}{h}  + \b_{h+1} \ont{h+1}{h+1} \,.
\end{align*}
All superscripts now match their corresponding subscripts and so we can write down an algorithm that is equivalent to the forward view in the sense that $\tt{t}{t} = \tth_t$ for all $t$, with
\beq \label{algTDbeta}
\begin{array}{ll}
\ee_{-1} \doteq 0, \text{ then }\ee_t \doteq \ee_{t-1} + \a_t\p_t( 1 - \p_t\tr\ee_{t-1}),  & t=0, \ldots, \tau-1, \\
\onth_{t+1} \doteq  \onth_t + \ee_t\left(\Z^{t+1}-\Z^t\right) +\a_t \p_t (\Z^t-\p_t\tr\onth_{t}) , & t=0, \ldots, \tau-1,\\
\tth_{t+1} \doteq \tth_t + \b_{t+1} ( \onth_{t+1} - \tth_t ) , & t=0, \ldots, \tau-1,
\end{array}
\eeq
where $\b_t \in [0,1]$ is the degree of trust we place in $\Z^t$. The first two lines compute the online weights, and are equal to the online backward view \eqref{algTDFinal} from the previous section.  The last line effectively computes a weighted running average the online weights, according to the sequence $\{\b_t\}_{t=1}^{\tau}$.

Algorithm \eqref{algTDbeta} subsumes the previous algorithms.  If $\b_t = 1$, $\forall t$, then the predictions are equal those of the online algorithm on each step.  If $\b_t = 0$ for $t \in \{ 1, \ldots, \tau-1\}$ and $\b_{\tau} = 1$, then the predictions are equal those of the offline algorithm.  As long as we trust the final outcome, such that $\b_{\tau} = 1$, then all these algorithms result in exactly the same weights at the end of the episode, and algorithm \eqref{algTDbeta} allows us to be flexible about how much we change the predictions during the episode, without requiring us to commit to either fully online or fully offline updates for the whole episode.

\section{Bootstrapping}\label{sec:persistence}
So far we have always fully trusted the actual final outcome.  All interim targets have been deemed irrelevant by the end of the episode, leaving no effect on the computed final weights.  There are cases in which we do not want to discard all interim targets. For instance, consider a stock that crashes down just before the end of the year. Certainly our updated predictions should include the possibility of such a crash, but we may not want to predict it will always crash just before the year's end.  Similarly, suppose it rains on a certain date for which we want to predict the weather.  It then seems wasteful to ignore the sunny weather on the days leading up to that date.  These are examples of cases in which the interim targets are almost as informative as the final outcome.  In some cases, an interim target may even be \emph{more} informative.  For instance, it may be due to a highly-trusted expert that takes into account all possible outcomes from that point in time.   Surely, this expert should not be ignored completely in favor of one random final outcome.

The general idea of updating predictions using other predictions, such as the interim targets, is called \emph{bootstrapping} \citep[see, e.g.,][]{SuttonBarto:1998}.
One way to obtain persistence of interim targets in our final predictions, and to achieve bootstrapping, is to drop the requirement in the previous section that we trust the final outcome fully, and allow $\b_{\tau} < 1$.  The eventual target for our updates is then a weighted average of the interim targets and the final outcome.  For instance, if $\b_t = 1/2$ for all $t \in \{ 1, \ldots, \tau \}$ the final updates will place a weight of $\b_{\tau} = 1/2$ on the final outcome, a weight of $(1 - \b_{\tau})\b_{\tau-1} = 1/4$ on the interim target immediate before then, and so on.

The notion of trust from the previous section applies a single degree of trust for an interim target uniformly to all prior predictions, but this is not always desirable.  Using this definition of trust, if we fully trust an interim target then it replaces all earlier interim targets.
However, even if an interim target if fully trusted for the most recent prediction, it may be inherently less trustworthy for earlier predictions.
To illustrate this, consider flipping a coin three times and predicting the total number of heads. The possible final outcomes are $0$, $1$, $2$, and $3$ heads.  If a trusted expert tells us before the first flip that the coin is fair, that is equivalent to observing a trustworthy interim target of $\Z^1 = 1.5$.  Suppose then the first two flips both result in heads, such that the only remaining possible final outcomes are $2$ and $3$.  If the coin is indeed fair, an interim target of $2.5$ would now be trustworthy target for the prediction made after observing two heads.  However, we would probably not want to replace the earlier interim target of $1.5$ for the first prediction.  Unfortunately, this is exactly what happens with the algorithm from the previous section.

This suggests a different notion of trust, based on the degree of trust we place in an interim target as a stand-in for the \emph{expected} final outcome rather than the \emph{actual} (random) outcome.  If an interim target precisely matches the expected outcome at that point in time, then later targets can then only be noisier or more specific but not more informative, and it should never be replaced by later targets.  

Suppose, concretely, that we fully trust $\Z^t$ under this new notion of trust.  Then, the update for the prediction made at time $t-1$ should disregard any targets that arrive later, including even the final outcome.  Conversely, if we do not trust $\Z^t$ at all, then it should leave no trace in the final updates.  More generally we can update towards $\Z^t$ with an intermediate degree of trust $\eta_t \in [0,1]$.  If the next interim target $\Z^{t+1}$ is trusted with degree $\eta_{t+1}$, we then update our prediction at time $t$ towards it with a total weight of $( 1 - \eta_t ) \eta_{t+1}$.  The update towards $\Z^{t+2}$ will get a total weight of $(1 - \eta_t) ( 1 - \eta_{t+1}) \eta_{t+2}$, and so on until we reach either the final outcome or the current data horizon.  The latest interim target (and the final outcome) is always trusted fully until we move to the next time horizon. Therefore, at horizon $h$ we always place any remaining weight on $\Z^h$ and update towards it with total weight $(1 - \eta_t) ( 1 - \eta_{t+1} ) \cdots ( 1 - \eta_{h-1} )$.

The corresponding total update to the prediction at time $t$ with a current horizon $h$ is then
\begin{align}
\t{t+1}{h} = \t{t}{h} 
& + \eta_{t+1} \a_t \p_t ( \Z^{t+1} - \p_t\tr\t{t}{h} ) \notag\\
& + (1 - \eta_{t+1}) \eta_{t+2} \a_t \p_t ( \Z^{t+2} - \p_t\tr\t{t}{h} ) \notag\\
& + (1 - \eta_{t+1}) ( 1- \eta_{t+2}) \eta_{t+3} \a_t \p_t ( \Z^{t+3} - \p_t\tr\t{t}{h} ) \notag\\
& + \ldots \notag \\
& + (1 - \eta_{t+1}) \cdots (1 - \eta_{h-2}) \eta_{h-1} \a_t \p_t ( \Z^{h-1} - \p_t\tr\t{t}{h} ) \notag\\
& + (1 - \eta_{t+1}) \cdots (1 - \eta_{h-1})  \a_t \p_t ( \Z^h - \p_t\tr\t{t}{h} ) \notag\\
\doteq \t{t}{h}
& + \a_t \p_t \left( \Zl{t}{h} - \p_t\tr\t{t}{h} \right) \,.\label{FV_inverse_trust}
\end{align}
where we have grouped the updates into a single update towards a combined target $\Zl{t}{h}$, just as in the previous section.
This update is perhaps more familiar when we change the notation slightly.  For all $t$, we define $\l_t \doteq 1 - \eta_t$ such that $\l_t$ essentially specifies to  what degree we \emph{dis}trust $\Z^t$.  The combined target is then defined as
\begin{align}
\Zl{t}{h}
& \doteq (1 - \l_{t+1}) \Z^{t+1} \notag\\
& \quad + \l_{t+1} (1 - \l_{t+2}) \Z^{t+2} \notag\\
& \quad \ldots \notag\\
& \quad + \l_{t+1} \cdots \l_{h-2} (1 - \l_{h-1}) \Z^{h-1}\notag\\
& \quad + \l_{t+1} \cdots \l_{h-1}\Z^{h}\,. \label{Zl_def}
\end{align}
This target $\Z^h_t$ is known as a $\l$-return \parencite{SuttonBarto:1998}.  The version that truncates at the current horizon $h$ was first proposed by \textcite{vanSeijen:2014}.
The total set of updates is
\begin{align}
\t{0}{t} &\doteq \th_0 \,,& & t = 0, \ldots, \tau \,;\notag\\
\t{t+1}{h}
&\doteq \t{t}{h} + \a_t \p_t ( \Zl{t}{h} - \p_t\tr\t{t}{h} ) \,, \label{FV_lambda}\\
&\doteq \F_t \t{t}{h} + \a_t \p_t \Zl{t}{h} \,, & & t = 0, \ldots, h-1\,;~~ h = 1, \ldots, \tau\,.\notag
\end{align}
If we ultimately distrust all interim targets, then $\l_t = 1$ for all $t$ and $\Zl{t}{h} = \Z^h$.  The algorithm then reduces to the online algorithm from Section~4.  Otherwise, at least some interim targets persist and contribute to the final weights.  At the other extreme, if we trust all interim targets, then $\l_t = 0$ for all $t$ and $\Zl{t}{h} = \Z^{t+1}$. Then, the updates reduce to single-step updates that only use the immediate next interim target.  In that case each update depends only on the immediate next time step and we can drop the superscript $h$ and the forward view \eqref{FV_lambda} reduces to an efficient span-independent algorithm
\[
\th_{t+1} \doteq \th_t + \a_t \p_t ( \Z^{t+1} - \p_t\tr\th_t ) \,,\quad t=0,\ldots,\tau\,.
\]
Apart from this special case, the forward view \eqref{FV_lambda} is computationally inefficient and we desire an efficient span-independent algorithm to get from $\t{h}{h}$ to $\t{h+1}{h+1}$.  In the previous section  we derived how a weight vector depends on any sequence of combined targets $\Zb{t}{h}$, independent on the definition of those targets.  We repeat the result of that derivation, as first given in \eqref{th_th_beta}, here for clarity:
\[
\t{t}{h} 
= \aa_{t-1} + \sum_{k=0}^{t-1} \F_{h-1} \cdots \F_{k+1} \a_k \p_k \Zl{k}{h} \qquad t = 0, \ldots, h-1\,;~~h = 1, \ldots, \tau\,.
\]
Because this equation holds regardless of the definition of $\Zl{t}{h}$, we can apply it to the current algorithm.  In particular we use it to try to find an efficient algorithm to go from  $\t{h}{h}$ to $\t{h}{h+1}$.  If this is possible, we can then use the update \eqref{FV_lambda} to go from $\t{h}{h+1}$ to $\t{h+1}{h+1}$.  We start by writing out the difference as
\begin{align}
\t{h}{h+1} - \t{h}{h}
& = \sum_{k=0}^{h-1} \F_{h-1} \cdots \F_{k+1} \a_k \p_k \Zl{k}{h+1}  ~~-~~ \sum_{k=0}^{h-1} \F_{h-1} \cdots \F_{k+1} \a_t \p_k \Zl{k}{h} \tag{$\aa_{h-1}$ cancels}\notag\\
& = \sum_{k=0}^{h-1} \F_{h-1} \cdots \F_{k+1} \a_k \p_k \left( \Zl{k}{h+1} - \Zl{k}{h} \right) \,.\label{step_dif_Zl}
\end{align}
The combined targets $\Zl{k}{h+1}$ and $\Zl{k}{h}$ share many terms: going back to \eqref{Zl_def} we can see that all interim targets up to $\Z^{h-1}$ will have the exact same multipliers.  These terms cancel, and the remaining difference is given by
\begin{align}
\Zl{k}{h+1} - \Zl{k}{h}
& = \underbrace{\l_{k+1}\cdots \l_{h-1} (1 - \l_h) \Z^h + \l_{k+1}\cdots \l_{h} \Z^{h+1}}_{\mbox{due to $\Zl{k}{h+1}$}} - \underbrace{\l_{k+1}\cdots \l_{h-1} \Z^h}_{\mbox{due to $\Zl{k}{h}$}} \notag\\
& = \l_{k+1}\cdots \l_{h} \left( \Z^{h+1} - \Z^h \right) \,.\label{Zl_dif}
\end{align}
We can then continue from \eqref{step_dif_Zl} with
\begin{align}
\t{h}{h+1} - \t{h}{h}
& = \sum_{k=0}^{h-1} \F_{h-1} \cdots \F_{k+1} \a_k \p_k \left( \Zl{k}{h+1} - \Zl{k}{h} \right) \notag\\
& = \sum_{k=0}^{h-1} \F_{h-1} \cdots \F_{k+1} \a_k \p_k \l_{k+1}\cdots \l_{h} ( \Z^{h+1} - \Z^h )  \tag{using \eqref{Zl_dif}}\\
& = \l_h \underbrace{\left( \sum_{k=0}^{h-1} \F_{h-1} \cdots \F_{k+1} \l_{k+1} \cdots \l_{h-1} \a_k \p_k \right) }_{\mbox{$\doteq \el_{h-1}$}} ( \Z^{h+1} - \Z^h ) \notag\\
& = \l_h \el_{h-1} ( \Z^{h+1} - \Z^h ) \,.\label{step1_lambda}
\end{align}
We again encounter the TD error $\Z^{h+1} - \Z^h$ and, more importantly, a new trace vector $\el_t$ that can be updated efficiently with
\begin{align}
\el_t
& \doteq \sum_{k=0}^{t} \F_{t} \cdots \F_{k+1} \l_{k+1} \cdots \l_{t} \a_k \p_k \notag\\
& =  \sum_{k=0}^{t-1} \F_{t} \cdots \F_{k+1} \l_{k+1} \cdots \l_{t} \a_k \p_k ~~+~~ \a_t \p_t \notag\\
& =  \l_t \F_t \underbrace{\sum_{k=0}^{t-1} \F_{t-1} \cdots \F_{k+1} \l_{k+1} \cdots \l_{t-1} \a_k \p_k}_{\mbox{$ \el_{t-1}$}} ~~+~~ \a_t \p_t \notag\\
& =  \l_t \F_t \el_{t-1} + \a_t \p_t \label{trace_lambda}\\
& =  \l_t  \el_{t-1} + \a_t \p_t ( 1 - \l_t \p_t\tr \el_{t-1} ) \,.\notag
\end{align}
This trace is similar to the one we encountered before, but with the difference that the value of the vector decays towards zero by multiplication with $\l_t$ on each step.  Predictions made prior to a fully trusted interim target (for which $\l_t = 0$) will never be affected by later interim targets because the trace vector is set to zero.  The extent to which the trace extends backward in time depends on the extent to which we have not yet trusted the corresponding interim targets.

We now combine the derived update from $\t{h}{h}$ to $\t{h}{h+1}$ with the update from $\t{h}{h+1}$ to $\t{h+1}{h+1}$ to derive a single efficient update, given by
\begin{align*}
\t{h+1}{h+1}
& = \F_h \t{h}{h+1} + \a_h \p_h \Z^{h+1} \tag{using \eqref{FV_lambda}}\\
& = \F_h \left( \t{h}{h} + \l_h \el_{h-1} ( \Z^{h+1} - \Z^h ) \right) + \a_h \p_h \Z^{h+1} \tag{using \eqref{step1_lambda}}\\
& = \F_h \t{h}{h} + \l_h \F_h \el_{h-1} ( \Z^{h+1} - \Z^h ) + \a_h \p_h \Z^{h+1} \\
& = \F_h \t{h}{h} + ( \el_h - \a_h \p_h ) \left( \Z^{h+1} - \Z^h \right) + \a_h \p_h \Z^{h+1} \tag{using $\l_h \F_h \el_{h-1} = \el_h - \a_h \p_h$, from \eqref{trace_lambda}}\\
& = \F_h \t{h}{h} + \el_h \left( \Z^{h+1} - \Z^h \right) + \a_h \p_h \Z^h \\
& = ( \I - \a_h \p_h \p_h\tr) \t{h}{h} + \el_h \left( \Z^{h+1} - \Z^h \right) + \a_h \p_h \Z^h \\
& = \t{h}{h} + \el_h \left( \Z^{h+1} - \Z^h \right) + \a_h \p_h \left(  \Z^h - \p_h\tr\t{h}{h} \right) \,.
\end{align*}
This concludes our derivation because the value of $\t{h+1}{h+1}$ is now defined fully in terms of the previous weights on the diagonal $\t{h}{h}$ and other quantities that are either directly available upon reaching our new data horizon $h+1$ or, in the case of the trace vector, can be computed with constant $O(n)$ computation per step.  The span-independent algorithm with persistent interim targets is
\beq \label{algTDlambda}
\begin{array}{ll}
\el_{-1} \doteq {\bm 0}, \text{ then }\el_t \doteq \l_t  \el_{t-1} + \a_t \p_t ( 1 - \l_t \p_t\tr \el_{t-1} ),  & t=1, \ldots, \tau\!-\!1, \\
\th_{t+1} \doteq  \th_t + \el_t \left(\Z^{t+1} - \Z^t\right) +\a_t \p_t (\Z^t-\p_t\tr\th_{t})  & t=0, \ldots, \tau\!-\!1.
\end{array}
\eeq
Compared to the online backward view \eqref{algTDFinal}, the only difference is the appearance of $\l_t$ in the update of the trace.  If $\l_t = 1$ for all $t$, we regain the online backward view precisely, demonstrating that the new algorithm is strictly more general.
In contrast to the averaging backward view \eqref{algTDbeta}, from the previous section, we see that for the notion of trust corresponding to $\l$ we do not need to maintain separate online and trusted weights.  Instead, the degree of trust is used to scale the trace vector $\el_t$ down accordingly.  If $\l_t < 1$ for any $t \in \{1, \ldots, \tau\}$, the corresponding interim targets have a lasting effect on the final weight vector and therefore for the first time we may obtain predictions that differ not just during the episode but also at its end.

\section{Combining two notions of trust and the emergence of averaged TD($\l$)}\label{sec:persistent_trust}

Algorithm \eqref{algTDlambda} is a strict generalization of the online algorithm \eqref{algTDFinal}, but it does not subsume the offline algorithm \eqref{algFinal}, or the averaging algorithm \eqref{algTDbeta} that switches smoothly between online and offline updates.  In this section, we combine the ideas from the last two sections to arrive at an algorithm that generalizes and subsumes all previous algorithms, thereby unifying all that came before into a single, general-purpose algorithm.

An offline version of the TD($\l$) algorithm can be obtained by using the online algorithm in \eqref{algTDlambda} to update an online weight vector $\onth_t$ and then defining the trusted weight vector $\tth_t$ to remain equal to the initial weights until the last step, at which time we replace them with the online weights.  An algorithm that switches smoothly between the offline and online cases can then be obtained similar to before, resulting in
\beq \label{algTDbetalambda}
\begin{array}{ll}
\ee_{-1} \doteq {\bm 0}, \text{ then }\ee_t \doteq \l_t  \el_{t-1} + \a_t \p_t ( 1 - \l_t \p_t\tr \el_{t-1} ),  & t=0, \ldots, \tau-1, \\
\onth_{t+1} \doteq  \onth_t + \ee_t \left(\Z^{t+1} - \Z^t\right) + \a_t \p_t (\Z^t-\p_t\tr\onth_{t}) , & t=0, \ldots, \tau-1, \\
\tth_{t+1} \doteq \tth_t + \b_{t+1} ( \onth_{t+1} - \tth_t ) , & t=0, \ldots, \tau-1.
\end{array}
\eeq
The first two lines are the online algorithm \eqref{algTDlambda}, from the previous section.
The last line is equal to the last line in the unified algorithm without persistency of interim targets, as given in \eqref{algTDbeta}, but now using the online weights that use persistent interim targets weighted according to $\l$-returns, as computed in the first two lines.
When $\l_t = 1$ for all $t$ we regain the averaging algorithm \eqref{algTDbeta} without persistent interim targets. When $\b_t = 1$ for all $t$ we regain the online algorithm \eqref{algTDlambda} with persistent interim targets.  So, we have again successfully unified all previously seemingly difference approaches to trust and have arrived at a single general algorithm that subsumes all that came before.

The merits of $\l$-returns are well known \parencite{Sutton:1988,SuttonBarto:1998} but the $\b$-weighting of the online weights is novel to this paper, and it is appropriate to discuss it in a little more detail.  So far we have considered only a single episode, but a major potential benefit of including $\b$ appears when we consider multiple episodes.  For clarity, consider the extreme case where all episodes last only a single step.  Then $\l$ plays no role because there is no interim within each episode; there is only a beginning and an end.  
If we use $m$ to enumerate episodes, such that $\Z_m$ is the true outcome of episode $m$, then the updates for single-step episodes are
\begin{align*}
\onth_{m+1} & = \onth_m + \a_m ( \Z_m - \p_m\tr\onth_m ) \p_m \,,\\
\tth_{m+1} & = \tth_m + \b_{m+1} ( \onth_{m+1} - \tth_m ) \,,
\end{align*}
where we used the fact that the final weights of episode $n$ are the first weights of episode $m+1$.  
Using $\b$, we can do something here that we cannot do with $\l$ alone: we can weight the relative impact of different episodes.  For instance, we can choose to keep the weights and the predictions stationary over multiple episodes, by setting $\b_m = 0$, to reduce the impact of the final outcomes observed in those episodes on our predictions.
Another possibility is to decay the trust, for instance according to $\b_m = \frac{1}{m}$, such that we trust the outcome of the first episode fully ($\b_1 = 1$) and then reduce the trust for each subsequent episode.  Such a choice of $\b$ makes sense if we view the trust we place in outcomes as being relative to the trust we place in the predictions we already have.  As our predictions improve over time, the outcomes become relatively less trustworthy.  With this definition of trust, the trusted weights are  the average of the weights of all previous episodes:
$
\tth_m = \frac{1}{m} \sum_{i=0}^m \onth_i
$.
This specific algorithm is interesting because the predictions according to the averages $\tth_m$ are known to converge to the optimal predictions faster than the predictions according to any sequence of online weights $\onth_m$ \parencite{Polyak:1992,Bach:2013}.
	This shows that the notion of trust as provided by $\b$ gives us something that cannot be obtained with $\l$ alone.  To our knowledge, algorithm \eqref{algTDbetalambda} is the first to generalize this idea of averaging online weights, in a principled fashion, to long-term predictions.


\section{Generalizing to cumulative returns and soft terminations}\label{sec:generalizations}
In this section, we discuss how to extend our algorithms to handle soft terminations and cumulative returns.  Both extensions generalize the episodic final-outcome setting considered above, and the algorithm we derive in this section will subsume all previous algorithms.

Often, we want to predict the cumulation of a signal $\{ \R_t \}_{t=0}^{\tau}$ rather than a single final outcome.  
In the episodic setting, with termination at time $\tau$, we then aim to predict
\[
\ZZ{t}{\tau} \doteq \R_{t+1} + \R_{t+2} + \ldots + \R_{\tau} \,,
\]
where, as before, $t$ is the time of the prediction.
In contrast to final outcomes, these cumulative outcomes depend on the time step of the prediction because at later time steps there will be less signal left to accumulate before the episode ends.  We call this time-dependent outcome the \emph{cumulative return}.

We may wish to update our predictions online, before observing the full cumulative return.  To do so, we need to define interim targets to temporarily take the place of the actual return. An interim target $\ZZ{t}{h}$ up to a horizon $h< \tau$ should in any case include the part of the signal that was already observed. In addition, we introduce a residual prediction $\P_h$ that stands in for the unseen part of the signal, from horizon $h$ to the end of the episode $\tau$.  The full interim target at time $t$ up to horizon $h$ is then
\[
\ZZ{t}{h} = \R_{t+1} + \R_{t+2} + \ldots + \R_h + \P_h \,.
\]
The residual prediction $\P_h$ may for instance be given by an external expert or by our own predictions at time $h$.  For now we are agnostic to its origin and consider the general case.
In any case, we define $\P_{\tau} = 0$ because at termination there is no remaining signal left to predict.  If all signals except the one coinciding with termination are zero, we regain the final-outcome setting where the last signal $\R_{\tau} = \Z$ takes the role of the final outcome. If we further define the interim target at each horizon to be equal to the residual prediction at that horizon, such that $\Z^h = \P_h$, we are back in the online final-outcome setting.  Therefore, cumulative returns are strictly more general than final outcomes.

So far, we have considered episodic predictions where each episode ends with a single final outcome that we wish to predict.  The algorithms extend naturally to multiple episodes by using the final weights of the episode as the initial weights of the next episode.  However, some predictive questions do not fit nicely into this strictly episodic format because they are better thought of as terminating \emph{softly} on each time step.


A soft termination is a conceptual and potentially partial termination of the signal.  Such a soft termination could represent a probability of termination, for instance when we want to take into account the probability of a robot breaking while learning in a simulation in which it never actually does.  Or the soft termination could represent a desire to trade off the imminence and the magnitude of a signal, for instance when we do not just want more money rather than less, but we also want it sooner rather than later.  In both cases the prediction is about a diminishing version of the `raw' signal (e.g., money). 
Here, we are not concerned with the potential reasons for using soft terminations and consider the general case where the termination of the prediction can vary per time step and be anywhere between full continuation and full termination.
Soft terminations allow us to ask more general predictive questions, and even to simultaneously consider multiple predictions that may resolve at different times.


Soft terminations can be modeled by using a continuation parameter $\g_t \in [0,1]$ to denote the amount of termination of our prediction upon reaching time $t$.  This quantity is often called a discount factor, because it discounts the impact of later outcomes compared to earlier ones.  If $\g_t = 1$, no termination happens at time $t$; if $\g_t = 0$, the prediction terminates fully, even if the trajectory may continue.  We consider general sequences of $\g_t$ and only require that eventually every prediction resolves completely, potentially asymptotically, such that $\prod_{i=t}^{\infty} \g_i = 0$ for all $t$.  The episodic setting considered in the previous sections is a special case where $\g_t = 0$ only if $t=\tau$ is the final step of the episode, and $\g_t = 1$ on all other steps.

We are now almost ready to formulate the target for a prediction about a discounted cumulative signal. In order to maintain full generality and compatibility with previous sections, we immediately include persistency of the residual predictions according to a sequence of $\{\l_t\}_{t=0}^{\infty}$ and trust of the resulting combined targets according to a sequence $\{\b_t\}_{t=0}^{\infty}$.
For any horizon $h>t$, the combined interim target for the prediction at time $t$ should include at least the immediate next signal $\R_{t+1}$.  Beyond this first signal we continue with $\g_{t+1}$ and observe the residual prediction $\P_{t+1}$.  We trust this prediction with degree $1 - \l_{t+1}$ as a stand-in for the expected cumulative discounted return, and so its total multiplier is $\g_{t+1} ( 1 - \l_{t+1} )$. If we trust this prediction fully, so that $\l_t = 0$, there is no need to continue further.  Otherwise, we continue for as much as we have not yet terminated according to both $\g_{t+1}$ and $\l_{t+1}$, so that the next signal $\R_{t+2}$ gets a total weight of $\g_{t+1} \l_{t+1}$.  This line of reasoning then continues until we reach the horizon of our current data at time $h$, at which point we place any remaining trust on the most recent residual prediction $\P_h$.
All together, this gives us a combined target for the prediction at time $t$ with data up to time $h>t$:
\begin{align}
\Zlg{t}{h}
& \doteq \R_{t+1} + \g_{t+1} ( 1 - \l_{t+1} ) \P_{t+1} \notag\\
& \quad {} + \g_{t+1} \l_{t+1} ( \R_{t+2} + \g_{t+2} ( 1 - \l_{t+2} ) \P_{t+2} ) \notag\\
& \quad {} + \g_{t+1} \l_{t+1} \g_{t+2} \l_{t+2} ( \R_{t+3} + \g_{t+3} ( 1 - \l_{t+3} ) \P_{t+3} ) \notag\\
& \quad ~~\vdots \notag\\
& \quad {} + \g_{t+1} \cdots \g_{h-2} \l_{t+1} \cdots \l_{h-2} ( \R_{h-1} + \g_{h-1} (1 - \l_{h-1}) \P_{h-1} ) \notag\\
& \quad {} + \g_{t+1} \cdots \g_{h-1} \l_{t+1} \cdots \l_{h-1} ( \R_h + \g_h \P_h ) \,.\label{Zlg_def_non_recursive}
\end{align}
This can be written recursively as
\begin{align}
\Zlg{t}{h} &\doteq \R_{t+1} + \g_{t+1} \big( ( 1 - \l_{t+1} )  \P_{t+1} + \l_{t+1} \Zlg{t+1}{h} \big) \,,\label{Zlg_def}\\
\text{and }~ \Zlg{t}{t} & \doteq \P_t \,, \label{Zlg_tt_def}
\end{align}
Each of these interim targets is trusted according to the trust associated with its horizon, $\b_h$.  Recall from Section \ref{sec:trust} that this trust propagates backward.  If we trust the last target $\Zlg{t}{h}$ fully, there is no need to consider earlier targets.  Otherwise, we multiply $\Zlg{t}{h}$ with its associated trust $\b_h$ and continue to the previous target with degree $(1 - \b_h)$.  This then continues until either we find a target we fully trust, or we reach the then-current predictions $\p_t\tr\tt{t}{t}$ at time $t$, when we made the prediction we are currently updating.  Our final interim targets are then given by
\begin{align}
\Zblg{t}{h}
\doteq {} & \b_h \Zlg{t}{h} \notag\\
& + ( 1 - \b_h ) \b_{h-1} \Zlg{t}{h-1} \notag\\
& + ( 1 - \b_h ) ( 1 - \b_{h-1} ) \b_{h-2} \Zlg{t}{h-2} \notag\\
& \vdots \notag\\
& + ( 1 - \b_h ) \cdots ( 1 - \b_{t+2} ) \b_{t+1} \Zlg{t}{t+1} \notag\\
& + ( 1 - \b_h ) \cdots ( 1 - \b_{t+1} ) \p_t\tr\tt{t}{t} \notag\\
= {} & \b_h \Zlg{t}{h} + ( 1 - \b_h ) \Zblg{t}{h-1} \label{Zblg_def} \\
\text{and }~ \Zblg{t}{t} \doteq {} & \p_t\tr\tt{t}{t} \label{Zblg_tt_def}
\end{align}
where $\tt{t}{t}$ are the trusted weights for time $t$ that we will compute.
The forward-viewing update is then given by
\begin{align}
\tt{0}{t}
& \doteq \th_0 && t = 0, \ldots, \tau \notag\\
\tt{t+1}{h}
& \doteq \tt{t}{h} + \a_t \p_t ( \Zblg{t}{h} - \p_t\tr\tt{t}{h} ) \,, && t = 0, \ldots\, h-1 \,;~~ h = 0, \ldots, \tau \,,\label{FV_blg}\\
& \doteq \F_t \tt{t}{h} + \a_t \p_t \Zblg{t}{h} \,, && t = 0, \ldots\, h-1 \,;~~ h = 0, \ldots, \tau \,.\notag
\end{align}

To find an efficient backward view, again we can start by writing down explicitly the value of $\tt{t}{h}$.  Following a derivation identical to the one for when we first consider trusted weights, as shown in \eqref{th_th_beta} but with the complex target $\Zblg{k}{h}$ replacing $\Zb{k}{h}$, we obtain
\beq\label{tt_blg}
\tt{t}{h} 
= \aa_{t-1} + \sum_{k=0}^{t-1} \F_{t-1} \cdots \F_{k+1} \a_k \p_k \Zblg{k}{h} \qquad t = 0, \ldots, h-1\,;~~h = 1, \ldots, \tau\,.
\eeq
As before, we can apply this to both $\tt{h}{h+1}$ and $\tt{h}{h}$ to find the difference
\begin{align}
\tt{h}{h+1} - \tt{h}{h}
& = \sum_{k=0}^{h-1} \F_{h-1} \cdots \F_{k+1} \a_k \p_k \Zblg{k}{h+1} ~~-~~ \sum_{k=0}^{h-1} \F_{h-1} \cdots \F_{k+1} \a_k \p_k \Zblg{k}{h} \tag{$\aa_{h-1}$ cancels}\\
& = \sum_{k=0}^{h-1} \F_{h-1} \cdots \F_{k+1} \a_k \p_k \bigg( \Zblg{k}{h+1} - \Zblg{k}{h} \bigg) \notag\\
& = \sum_{k=0}^{h-1} \F_{h-1} \cdots \F_{k+1} \a_k \p_k \bigg( \b_{h+1} \Zlg{k}{h+1} + ( 1 - \b_{h+1} ) \Zblg{k}{h} - \Zblg{k}{h} \bigg)\tag{from \eqref{Zblg_def}}\\
& = \sum_{k=0}^{h-1} \F_{h-1} \cdots \F_{k+1} \a_k \p_k \b_{h+1} \bigg( \Zlg{k}{h+1} - \Zblg{k}{h} \bigg) \notag\\
& = \b_{h+1} \sum_{k=0}^{h-1} \F_{h-1} \cdots \F_{k+1} \a_k \p_k \Zlg{k}{h+1} ~~-~~ \b_{h+1} \underbrace{\sum_{k=0}^{h-1} \F_{h-1} \cdots \F_{k+1} \a_k \p_k \Zblg{k}{h}}_{\mbox{$\doteq \tt{h}{h} - \aa_{h-1}$, from \eqref{tt_blg}}} \notag\\
& = \b_{h+1} \left( \aa_{h-1} + \sum_{k=0}^{h-1} \F_{h-1} \cdots \F_{k+1} \a_k \p_k \Zlg{k}{h+1} \right) ~~-~~ \b_{h+1} \tt{h}{h} \,.\label{step1_blg}
\end{align}
The first part of this result, within the brackets, looks familiar: notice the similarity to \eqref{tt_blg}.  The term can be interpreted as the intermediate result $\ont{h}{h+1}$ of a forward view with targets $\Zlg{t}{h}$, as defined in \eqref{Zlg_def}, and updates defined by
\beq\label{FV_lg}
\begin{array}{ll}
\ont{0}{t}
\doteq \onth_0 & t = 0, \ldots, \tau \\
\ont{t+1}{h}
\doteq \ont{t}{h} + \a_t \p_t ( \Zlg{t}{h} - \p_t\tr\ont{t}{h} ) \,, & t = 0, \ldots\, h-1 \,;~~ h = 0, \ldots, \tau \,.
\end{array}
\eeq
This is an online algorithm, comparable to the algorithm for $\l$-returns derived before, but implicitly including cumulative returns and discounting through the definition of $\Zlg{t}{h}$.  The complete forward view \eqref{FV_blg} can then be interpreted as switching smoothly between this online algorithm and an offline algorithm that in the extreme can delay updating the predictions indefinitely.
Analogous to the interim weights shown in \eqref{tt_blg}, the interim weights of this online algorithm satisfy
\beq\label{th_lg}
\ont{t}{h} 
= \aa_{t-1} + \sum_{k=0}^{t-1} \F_{t-1} \cdots \F_{k+1} \a_k \p_k \Zlg{k}{h} \qquad t = 0, \ldots, h\,;~~h = 1, \ldots, \tau\,.
\eeq
We can then continue our derivation from \eqref{step1_blg} with
\begin{align}
\tt{h}{h+1} - \tt{h}{h}
& = \b_{h+1} \underbrace{\left( \aa_{h-1} + \sum_{k=0}^{h-1} \F_{h-1} \cdots \F_{k+1} \a_k \p_k \Zlg{k}{h+1}\right)}_{\mbox{$=\ont{h}{h+1}$, using \eqref{th_lg}}} ~~-~~ \b_{h+1} \tt{h}{h} \notag\\
& = \b_{h+1} ( \ont{h}{h+1} - \tt{h}{h} ) \,.\notag
\end{align}
This implies that if we have $\tt{h}{h}$ and $\ont{h}{h+1}$ we can then compute $\tt{h}{h+1}$ efficiently using
\beq\label{tt_blg_hhnext}
\tt{h}{h+1} = (1 - \b_{h+1} ) \tt{h}{h} + \b_{h+1} \ont{h}{h+1} \,.
\eeq
For now, we ignore the question of how to obtain $\ont{h}{h+1}$ and first focus on how to go from $\tt{h}{h+1}$ to $\tt{h+1}{h+1}$ and, resultingly, from $\tt{h}{h}$ to $\tt{h+1}{h+1}$.
We can now derive
\begin{align*}
\tt{h+1}{h+1}
& = \F_h \tt{h}{h+1} + \a_h \p_h \Zblg{h}{h+1} \tag{using \eqref{FV_blg}} \\
& = \F_h \bigg(\b_{h+1} \ont{h}{h+1} + (1 - \b_{h+1}) \tt{h}{h} \bigg) + \a_h \p_h \Zblg{h}{h+1} \tag{using \eqref{tt_blg_hhnext}}\\
& = \F_h \bigg(\b_{h+1} \ont{h}{h+1} + (1 - \b_{h+1}) \tt{h}{h} \bigg) + \a_h \p_h \bigg(\b_{h+1} \Zlg{h}{h+1} + ( 1 - \b_{h+1} ) \Zblg{h}{h}\bigg) \tag{using \eqref{Zblg_def}}\\
& = \F_h \bigg(\b_{h+1} \ont{h}{h+1} + (1 - \b_{h+1}) \tt{h}{h} \bigg) + \a_h \p_h \bigg(\b_{h+1} \Zlg{h}{h+1} + ( 1 - \b_{h+1} )\p_h\tr\tt{h}{h} \bigg) \tag{using \eqref{Zblg_tt_def}}\\
& = \b_{h+1} \underbrace{\bigg( \F_h \ont{h}{h+1} + \a_h \p_h \Zlg{h}{h+1} \bigg)}_{\mbox{$\doteq \ont{h+1}{h+1}$, using \eqref{FV_lg}}} ~+~ ( 1 - \b_{h+1} ) \underbrace{\bigg( \F_h \tt{h}{h} + \a_h \p_h \p_h\tr\tt{h}{h} \bigg)}_{\mbox{$=\tt{h}{h}$}} \tag{regrouping}\\
& = \b_{h+1} \ont{h+1}{h+1} + ( 1 - \b_{h+1} ) \tt{h}{h} \,.
\end{align*}
We see that $\ont{h}{h+1}$ is not needed and instead we can use $\ont{h+1}{h+1}$.  Therefore, if $\ont{h+1}{h+1}$ can be computed with constant computation independent of span, then $\tt{h+1}{h+1}$ can be computed from $\tt{h}{h}$ efficiently as well, using
\begin{equation}\label{blg}
\tt{h+1}{h+1} = (1 - \b_{h+1}) \tt{h}{h} + \b_{h+1} \ont{h+1}{h+1} \,.
\end{equation}

It now remains to be shown that the online weights $\ont{h+1}{h+1}$ can be computed efficiently. We start with the difference between $\ont{h}{h+1}$ and $\ont{h}{h}$ for which we can use \eqref{th_lg} to obtain
\begin{align}
\ont{h}{h+1} - \ont{h}{h}
& = \sum_{k=0}^{h-1} \F_{h-1} \cdots \F_{k+1} \a_k \p_k \left( \Zlg{k}{h+1} - \Zlg{k}{h} \right)\label{step0_lg}\,.
\end{align}
We can find the difference $\Zlg{k}{h+1} - \Zlg{k}{h}$ from the definition in \eqref{Zlg_def_non_recursive}.  All terms not involving $\P_h$, $\R_{h+1}$ or $\P_{h+1}$ cancel, leaving us with
\begin{align*}
\Zlg{k}{h+1} - \Zlg{k}{h}
& = \underbrace{\g_{k+1}\cdots\g_h \l_{k+1}\cdots\l_{h-1}\left( ( 1 - \l_h )  \P_h + \l_h \R_{h+1} + \l_h\g_{h+1} \P_{h+1} \right)}_{\mbox{from $\Zlg{k}{h+1}$}}\\
& \quad {} - \underbrace{\g_{k+1}\cdots\g_h\l_{k+1}\cdots\l_{h-1}\P_h}_{\mbox{from $\Zlg{k}{h}$}} \\
& = \g_{k+1}\cdots\g_h\l_{k+1}\cdots\l_{h-1}\left( - \l_h \P_h + \l_h \R_{h+1} + \l_h\g_{h+1} \P_{h+1} \right)\\
& = \g_{k+1}\cdots\g_{h}\l_{k+1}\cdots\l_{h} \left( \R_{h+1} + \g_{h+1} \P_{h+1} - \P_h \right) \,.
\end{align*}
Here we see the emergence of a general form of the classical temporal-difference error for cumulative discounted returns: $\d_h \doteq \R_{h+1} + \g_{h+1} \P_{h+1} - \P_h$.\footnote{If we use our current predictions as interim targets, the TD error is $\d_t = \R_{t+1} + \g_{t+1} \p_{t+1}\tr \th_t - \p_t\tr\th_{t-1}$. This TD error uses the weights at two consecutive time steps and is therefore slightlt different from the classic TD error defined, using only the current weights $\th_t$, as $\d_t = \R_{t+1} + \g_{t+1} \p_{t+1}\tr \th_t - \p_t\tr\th_t$. The difference is important to achieve exact equivalence, although it is also possible to rewrite the new algorithms to use a more standard TD error.}
Using this, we can now continue from \eqref{step0_lg} with
\begin{align}
\ont{h}{h+1} - \ont{h}{h}
& = \sum_{k=0}^{h-1} \F_{h-1} \cdots \F_{k+1} \a_k \p_k \left( \Zlg{k}{h+1} - \Zlg{k}{h} \right) \notag\\
& = \sum_{k=0}^{h-1} \F_{h-1} \cdots \F_{k+1} \a_k \p_k \g_{k+1}\cdots\g_{h}\l_{k+1}\cdots\l_{h}\d_h  \notag\\
& = \g_h\l_h\underbrace{\left( \sum_{k=0}^{h-1} \g_{k+1}\cdots\g_{h-1}\l_{k+1}\cdots\l_{h-1}\F_{h-1} \cdots \F_{k+1} \a_k \p_k \right)}_{\mbox{$\doteq \egl_{h-1}$}} \d_h  \notag\\
& = \g_h\l_h\egl_{h-1} \d_h \,,\label{step1_lg}
\end{align}
where, similar to before, the trace vector $\egl_t$ can be updated recursively according to
\begin{align}
\egl_t
& \doteq \sum_{k=0}^{t} \g_{k+1} \cdots \g_t \l_{k+1} \cdots \l_{t} \F_{t} \cdots \F_{k+1} \a_k \p_k \notag\\
& =  \sum_{k=0}^{t-1} \g_{k+1} \cdots \g_t \l_{k+1} \cdots \l_{t} \F_{t} \cdots \F_{k+1} \a_k \p_k ~~+~~ \a_t \p_t \notag\\
& =  \g_t \l_t \F_t \sum_{k=0}^{t-1} \g_{k+1} \cdots \g_{t-1} \l_{k+1} \cdots \l_{t-1} \F_{t-1} \cdots \F_{k+1} \a_k \p_k ~~+~~ \a_t \p_t \notag\\
& =  \g_t \l_t \F_t \egl_{t-1} + \a_t \p_t \label{trace_lg}\\
& =  \g_t \l_t  \egl_{t-1} + \a_t \p_t ( 1 - \g_t \l_t \p_t\tr \egl_{t-1} ) \,.\notag
\end{align}
The trace now decays according to both $\l$ and $\g$.
Using this trace vector, we can compute $\ont{h}{h+1}$ efficiently from $\ont{h}{h}$ using \eqref{step1_lg}.  We now combine this with the final step to $\ont{h+1}{h+1}$ from $\ont{h}{h+1}$ and derive
\begin{align*}
\ont{h+1}{h+1}
& = \F_h \ont{h}{h+1} + \a_h \p_h \Zlg{h}{h+1} \tag{using \eqref{FV_lg}}\\
& = \F_h \ont{h}{h+1} + \a_h \p_h \left( \R_{h+1} + \g_{h+1} \P_{h+1} \right) \tag{using \eqref{Zlg_def}, \eqref{Zlg_tt_def}}\\
& = \F_h \left( \ont{h}{h} + \g_h\l_h\egl_{h-1}\left( \R_{h+1} + \g_{h+1} \P_{h+1} - \P_h \right)\right) + \a_h \p_h \left( \R_{h+1} + \g_{h+1} \P_{h+1} \right) \tag{using \eqref{step1_lg}}\\
& = \F_h \ont{h}{h} + \g_h\l_h \F_h \egl_{h-1} \left( \R_{h+1} + \g_{h+1} \P_{h+1} - \P_h \right) + \a_h \p_h \left( \R_{h+1} + \g_{h+1} \P_{h+1} \right) \\
& = \F_h \ont{h}{h} + ( \egl_h - \a_h \p_h ) \left( \R_{h+1} + \g_{h+1} \P_{h+1} - \P_h \right) + \a_h \p_h \left( \R_{h+1} + \g_{h+1} \P_{h+1} \right) \tag{using $\g_h\l_h \F_h \egl_{h-1} = \egl_h - \a_h \p_h$, from \eqref{trace_lg}}\\
& = \F_h \ont{h}{h} + \egl_h \left( \R_{h+1} + \g_{h+1} \P_{h+1} - \P_h \right) + \a_h \p_h \P_h \\
& = ( \I - \a_h \p_h \p_h\tr) \ont{h}{h} + \egl_h \left( \R_{h+1} + \g_{h+1} \P_{h+1} - \P_h \right) + \a_h \p_h \P_h \\
& = \ont{h}{h} + \egl_h \left( \R_{h+1} + \g_{h+1} \P_{h+1} - \P_h \right) + \a_h \p_h \left(  \P_h - \p_h\tr\ont{h}{h} \right) \\
& = \ont{h}{h} + \egl_h \d_h + \a_h \p_h \left(  \P_h - \p_h\tr\ont{h}{h} \right) \,.
\end{align*}
All weights again have matching sub- and superscripts and an equivalent TD algorithm, in the sense that $\onth_t = \ont{t}{t}$ and $\tth_t = \tt{t}{t}$, for the fully general case including cumulative discounted returns is given by
\beq \label{algTD_blg}
\begin{array}{ll}
\egl_{-1} \doteq {\bm 0}, \text{ then }\egl_t \doteq \g_t\l_t  \egl_{t-1} + \a_t \p_t ( 1 - \g_t\l_t \p_t\tr \egl_{t-1} ) ,  & t=0, \ldots, \tau-1, \\
\onth_{t+1} \doteq  \onth_t + \egl_t \d_t +\a_t \p_t (\P_t-\p_t\tr\onth_{t}) , & t=0, \ldots, \tau-1,\\
\tth_{t+1} \doteq \tth_t + \b_{t+1} ( \onth_{t+1} - \tth_t ) , & t=0, \ldots, \tau-1.
\end{array}
\eeq
The first two lines constitute the online algorithm we just derived.  The last line is from \eqref{blg} and extends this algorithm to include smooth switching between offline and online updates.  If $\b_t = 1$ for all $t$, the algorithm reduces to a variant of TD($\l$) known as true online TD($\l$) \parencite{vanSeijen:2014}, but extended to include general, potentially non-constant, sequences of $\{ \a_t \}$, $\{ \g_t \}$ and $\{ \l_t \}$.  The extension to averaging according to $\b_t$ is new to this paper.

Soft termination generalizes the episodic setting we considered previously.  This means that as far as the learning update is concerned, we do not have to treat steps on which the process actually terminates and restarts for a new episode in any special way.  To see how this works, we first renumber the time steps on consecutive episodes: if the first episode ends at time $\tau$, the initial time step of the second episode will be taken to be $\tau$ rather than $0$.  If the second episode lasts $\tau'$ steps, the third episode is then taken to begin on $\tau + \tau'$, and so on.  Together with the requirement that $\g_{\tau}=0$ on every actual termination, this is sufficient to get updates that are completely equivalent to treating the subsequent episodes completely separately.  Notice that the update on termination at some time $\tau$, and resulting in $\th_{\tau}$, uses the residual prediction on termination $\P_{\tau}$ only when multiplied with $\g_{\tau}$.  Previously we required that $\P_{\tau}=0$ because there is no further signal to predict.  This is now no longer necessary, because $\g_{\tau}=0$ already fulfills the requirement that $\g_{\tau}\P_{\tau} = 0$.  If for instance we use our current predictions, such that $\P_{t+1} \doteq \p_{t+1}\tr\th_t$ for all $t$, we can simply keep the update as is even though $\P_{\tau}$ will then be the prediction for the cumulative return of the next episode, because $\p_{\tau}$ is now defined to be its first feature vector.
Therefore, both hard and soft termination can be handled seamlessly using algorithm \eqref{algTD_blg}, and that will be our final, general algorithm.

\section{Convergence Analysis}\label{sec:analysis}

The algorithm \eqref{algTD_blg} differs from related earlier algorithms such as TD($\l$) in a few subtle but important ways. The most notable differences are the updates to the traces and the averaging due to $\b_t$. Known results on convergence therefore do not automatically transfer to this new algorithm and it is appropriate to take a moment to analyze it.


The convergence of the trusted weights $\tth$ depends on the convergence of the online weights $\onth$ and so we must investigate these jointly.  The online weights, in turn, depend on the sequences of parameters and residual predictions that are supplied.  We want our analysis to be general, which means we want to be able to handle general sequences of discounts $\{\g_t\}_{t=1}^{\infty}$, persistency parameters, $\{\l_t\}_{t=1}^{\infty}$, and residual predictions $\{ \P_t \}_{t=1}^{\infty}$.  Naturally, if any of these can change completely arbitrarily, we can have no hope of converging to any predeterminable solution.  Therefore, as in \textcite{Sutton:2014}, we allow the features, discounts and persistency parameters to be stationary functions of an underlying unobserved state, such that $\p_t \doteq \p(S_t)$, $\g_t \doteq \g(S_t)$ and $\l_t \doteq \l(S_t)$ for some fixed functions $\p : \mathcal{S} \to \mathbb{R}^n$, $\g : \mathcal{S} \to [0,1]$ and $\l : \mathcal{S} \to [0,1]$, where $\mathcal{S}$ is a state space and $S_t\in\mathcal{S}$ is the, unobserved, state of the world at time $t$.  We assume there is a steady-state distribution over these states such that all expectations used below are well-defined with respect to a distribution over states defined by $\lim_{t\to\infty}\Pr(S_t = s)$.  This setting generalizes the more standard approach where $\g_t = \g$ and $\l_t = \l$ are constants, because now these parameters can still change over time, but it avoids the possibility of arbitrary non-stationarity that would ruin convergence.

We first consider convergence when the residual predictions are also due to a fixed function of state, for instance because they are due to otherwise stationary experts or oracles.
\begin{theorem}
Let $\R_t\doteq\R(S_t)$, $\p_t\doteq\p(S_t)$, $\g_t\doteq\g(S_t)$, $\l_t\doteq\l(S_t)$ and $\P_t\doteq\P(S_t)$ all be fixed functions of (unobserved) states $S_t \in \mathcal{S}$, with a stable steady-state distribution $d$.
Then, if $\sum_{t=0}^{\infty} \a_t = \sum_{t=0}^\infty \b_t = \infty$, and $\sum_{t=0}^\infty \a_t^2 < \infty$, algorithm \eqref{algTD_blg} converges almost surely to the fixed-point solution
\[
\tth_* \doteq \E{ \p_t\p_t }^{-1}  \E{ \Zlg{t}{\infty} \p_t } \,.
\]
\end{theorem}
\begin{proof}
We start by analyzing the online weights $\onth_t$.  Because of the equivalence of the forward and backward views, we can investigate the forward view, which is easier to analyze.  In other words, instead of investigating $\lim_{t\to\infty} \onth_t$ as updated through \eqref{algTD_blg},  we investigate $\lim_{t\to\infty} \ont{t}{t}$ as updated through \eqref{FV_lg}.  By construction, the end result is exactly the same.  The asymptotic forward view as the horizon goes to infinity is
\begin{align*}
& \ont{t+1}{\infty}
= \ont{t}{\infty} + \a_t \p_t ( \Zlg{t}{\infty} - \p_t\tr\ont{t}{\infty} ) \,,&&t=0,\ldots\,,\\
\text{where }\,
& \Zlg{t}{\infty} 
\doteq \lim_{h\to\infty} \Zlg{t}{h}\,,&&t=0,\ldots\,.
\end{align*}
Because $\Zlg{t}{h}$ does not depend on the weights, this is a standard stochastic gradient-descent update $\onth_{t+1} = \onth_t - \a_t \nabla_{\onth} l(\onth) \rvert_{\onth_t}$ on the quadratic loss function
\[
l(\onth)\doteq \E{ ( \Zlg{t}{\infty} - \p_t\tr\onth )^2 }\,.
\]
If the step sizes are suitably chosen, for instance such that $\sum_{t=0}^\infty \a_t = \infty$ and $\sum_{t=0}^{\infty} \a_t^2 < \infty$ \citep{Robbins:1951}, and if the means and variances of $\Zlg{t}{\infty}$ and $\p_t$ are well-defined and bounded for all $t$, this update converges to the fixed-point solution $\th_*$ that minimizes the quadratic loss \citep[cf.][]{Kushner:2003}, such that
\[
\lim_{t\to\infty} \onth_t = \th_* \doteq \E{ \p_t\p_t\tr }^{-1} \E{ \p_t \Zlg{t}{\infty} } \,.
\]
It is straightforward to see that $\tth_t$ will have the same limit; it suffices to have $\sum_{t=0}^\infty \b_t = \infty$.
\end{proof}

Although convergence is already guaranteed when $\b_t = 1$ for all $t$, recent work has shown that for similar stochastic gradient algorithms the optimal rate of convergence is attained if $\b_t$ decreases much faster than $\a_t$, specifically when $\b_t = O(t^{-1})$ while $\a_t = \a$ for some constant $\a$ \citep{Bach:2013}.  More generally, it seems likely that convergence also holds if $\sum_{t=0}^{\infty} \b_t^2 < \infty$ and $\sum_{t=0}^\infty \a_t^2 = \infty$.  The observation that $\b_t$ should perhaps decrease over time for faster learning may seem at odds with our introduction of this parameter as a degree of trust.  However, these two views are quite compatible if we consider $\b_t$ to be the degree of trust we place in the online updates relative to the trust we place in our current predictions due to the trusted weights.  When the trust in the predictions increases over time, the relative trust in the inherently noisy online targets should then decrease.

Although Theorem 1 is already fairly general, it does not cover the important case when the residual predictions additionally depend on the weights we are updating.  It makes sense to use the predictions we trust most and therefore we now consider what happens when $\P_t \doteq \p_t\tr\tth_{t-1}$.  Notice that we have to use $\tth_{t-1}$ rather than $\tth_t$, because $\P_t$ is used in the computation of $\tth_t$ and so the latter is not yet available when we compute $\P_t$.  The analysis of this case is more complex than the previous one, because $\P_t$ is no longer a constant function of state.  This means the update is no longer a standard gradient-descent update on a quadratic loss, because the target $\Zlg{t}{h}$ for the online weights itself depends on the trusted weights that we are simultaneously updating.

The results by \citet{Bach:2013} on stochastic gradient descent indicate that perhaps the most interesting case is where $\b_t$ decreases faster than $\a_t$, such that $\lim_{t\to\infty} \b_t/\a_t = 0$. This suggests an analysis on two time scales is appropriate.
\begin{theorem}
Let $\R_t\doteq\R(S_t)$, $\p_t\doteq\p(S_t)$, $\g_t\doteq\g(S_t)$ and $\l_t\doteq\l(S_t)$ all be fixed bounded functions of (unobserved) states $S_t \in \mathcal{S}$, with a stable steady-state distribution $d$.  Define $\P_t \doteq \p_t\tr\tth_{t-1}$.
Then, if $\sum_{t=0}^{\infty} \a_t = \sum_{t=0}^\infty \b_t = \infty$, $\sum_{t=0}^\infty \a_t^2 < \infty$, and $\lim_{t\to\infty}\frac{\b_t}{\a_t} = 0$, algorithm \eqref{algTD_blg} converges almost surely to the TD fixed-point solution $\tth_*$ that minimizes the mean-squared projected Bellman error \citep{Sutton:2008,Sutton:2009}, such that
\beq\label{MSPBE}
\E{ ( \Z_t(\tth_*) - \p_t\tr\tth_*) \p_t\tr } \E{ \p_t\p_t }^{-1}  \E{ \p_t ( \Z_t(\tth_*) - \p_t\tr\tth_*) } = 0 \,,
\eeq
where
\[
\Z_t(\th) \doteq \R_{t+1} + \g_{t+1} ( 1 - \l_{t+1} ) \p_t\tr\th + \Z_t(\th) \,,~~\forall \th \,.
\]
\end{theorem}
\begin{proof}
In two-time-scale analyses, we are allowed to analyze the faster updates as if the slower updates have stopped.  This means that in analyzing the updates to the online weights $\onth_t$, we can assume the trusted weights $\tth_t$ are constant to analyze where $\onth_t$ converges towards as a function of the stationary $\tth_t$.  On the other hand, when we analyze the slower updates to the trusted weights $\tth_t$ we are allowed to assume the faster updates to the online weights converge completely between each two steps.  For more detail on analyzing stochastic approximations on two times scales, we refer to \citet{Borkar:1997}, \citet{Borkar:2008}, \citet{Kushner:2003}, and \citet{Konda:2004}.

We first analyze the convergence of the faster updates to the online weights, where we can assume that the trusted weights are stationary at some value $\tth$.  Then, using $\P_t \doteq \p_t\tr\tth$, the targets for the updates of the forward view are
\begin{align*}
\Zlg{t}{t}(\tth)
& = \p_t\tr\tth \,,\\
\Zlg{t}{h}(\tth)
& = \R_{t+1} + \g_{t+1} ( 1 - \l_{t+1} ) \p_t\tr\tth + \g_{t+1} \l_{t+1} \Zlg{t+1}{h}(\tth) \,,
\end{align*}
where we have extended the notation slightly to make the dependence of $\Zlg{t}{h}(\tth)$ on $\tth$ explicit.  Notice that the residual predictions on each time step depend on the same stationary trusted weights $\tth$.
Because of the assumed stationarity of $\tth$, the updates to the weights $\ont{t}{h}$ of the forward view can again be considered standard stochastic gradient updates and therefore these weights converge towards the fixed point $\onth_*(\tth)$, where again we make the dependence on $\tth$ explicit, defined by
\[
\onth_*(\tth) = \E{ \p_t \p_t\tr }^{-1} \E{ \p_t \Zlg{t}{\infty}(\tth) } \,,
\]
where $\Zlg{t}{\infty}(\tth) = \lim_{h\to\infty} \Zlg{t}{h}(\tth)$ denotes the limit of the target of the update as the horizon grows to infinity.  In an episodic setting, $\Zlg{t}{\infty} = \Zlg{t}{\tau}$ for all $t<\tau$, where $\tau$ denotes the first termination after time $k$.  More in general, $\Zlg{t}{\infty}$ is always well-defined because we require $\prod_{t=0}^{\infty} \g_t = 0$.

For the analysis of the slower updates to $\tth_t$, we can now assume the faster time scale has already converged to its fixed point $\onth_*(\tth_t)$ for the current weights.  Therefore, we analyze the update
\begin{align*}
\tth_{t+1}
& = \tth_t + \b_{t+1} ( \onth_*(\tth_t) - \tth_t ) \\
& = \tth_t + \b_{t+1} ( \E{ \p_k \p_k\tr }^{-1} \E{ \p_k \Zlg{k}{\infty}(\tth_t) } - \tth_t ) \,.
\end{align*}
This is a stochastic-approximation update that, under the conditions that $\sum_{t=0}^\infty \b_t = \infty$ and $\sum_{t=0}^\infty \b_t^2 < \infty$, converges almost surely to the fixed point $\tth_*$ that satisfies
\[
\tth_* = \E{ \p_t \p_t\tr }^{-1} \E{ \p_t \Zlg{t}{\infty}(\tth_*) } \,.
\]
If we multiply both sides with $\E{ \p_t\p_t\tr }$, this implies that
$
\E{ \p_t \p_t\tr \tth_* } = \E{ \p_t \Zlg{t}{\infty}(\tth_*) }
$
and therefore, by moving both terms to the same side and then multiplying with $\E{ \p_t\p_t\tr }^{-1}$,
\[
\E{ \p_t \p_t\tr }^{-1} \E{ \p_t ( \Zlg{t}{\infty}(\tth_*) - \p_t\tr\tth_*) } = {\bm 0} \,.
\]
It follows immediately that $\tth_*$ minimizes the mean-squared projected Bellman error completely to zero, as desired.
\end{proof}

\section{Discussion}\label{sec:discussion}
In this paper, we have considered how to answer predictive questions with algorithms that use constant computation per time step that is proportional to the number of learned weights, and 
that is independent of the span of the prediction.  We considered both final and cumulative outcomes, under online and offline updating, with and without persistency of the residual predictions we encounter during an episode, and with hard and soft termination.  In the end, we obtained a single general algorithm that can be used for all these different predictive questions, which is shown in \eqref{algTD_blg}.  This algorithm is guaranteed to be convergent under typical, fairly mild, technical conditions.


Some extensions remain for future work.  In particular, we have not considered how different policies of behavior can influence our predictions, and as a result have not talked about the problem of control in which the goal is to find the optimal policy for a given (reward) signal.
Our analysis already extends naturally to the prediction of action values, from which control policies can be easily distilled.  Then, using a form of policy iteration \citep{Bellman:1957,Howard:1960}, we can repeatedly switch between estimating and improving the policy to tackle the problem of optimal control.  However, to properly and fully include adaptable policies, we would in addition need to carefully consider the problem of learning off-policy, about action-selection policies that differ from the one used to generate the data \citep{SuttonBarto:1998}.  This is consistent but orthogonal to the ideas outlined in this paper, and such off-policy predictions (including those about the greedy and, ultimately, optimal policy) are learnable through a proper use of rejection sampling, as in Q-learning, or importance sampling \citep{Precup:2000, Precup:2001, Maei:2011,Sutton:2014,vanHasselt:2014,Mahmood:2014}.

All algorithms considered in this paper are in a sense descendent from a linear stochastic gradient, or LMS, update.  The main idea of span-independent computation is more general and can be applied quite naturally to other settings, including for instance non-linear functions such as deep neural networks \citep{LeCun:2015,Mnih:2015} or to quadratic-time linear-function algorithms as in LSTD \citep{Bradtke:1996}.  Not all updates may have fully equivalent span-independent counterparts, but even then it may be more important to be independent of span than to be exactly equivalent.

\printbibliography

\end{document}

One might hope to find a better way, wherein on each step one could decide how much one trusted that step's interim residual prediction (or its overall cumulative-outcome prediction) and use a step-by-step weighting. In fact, there is a simple way to achieve such a sliding confidence in the interim predictions, utilizing the existing freedom to set the interim predictions arbitrarily. For example, we can set them such that the error on each step (except the last of the episode) is zero, and still on the last step recover the effect of doing exact updates on each step of the way. And this can be generalized to smoothly scale, or throttle, the effect of the interim predictions. Let $\b_t>0$ be a scale factor on the effect of the interim prediction for time $t$. We set $\b_t=0$ to say we want no effect of the interim prediction and $\b_t=1$ to say we want to have full effect. 
I believe we get exactly the desired sliding from no effect of an arbitrary interim prediction, now denoted $\bar\P_t, t=1, \ldots, \tau$ (e.g., $\bar\P_t=\p_t\tr\th_{t-1}$), to full effect, by defining 
\beq
\P_1\doteq\bar\P_1, \text{~~then~~} \P_{t} \doteq \b_t\bar\P_{t} + (1-\b_t)(\P_{t-1}-\R_{t}), \qquad t=2, \ldots, \tau, 
\eeq
and $\P_{\tau}\doteq 0$ as always.
Is this correct? I wish I could prove it correct.
